\documentclass[11pt,letterpaper]{article}

\usepackage{amsmath,amssymb,amsthm}
\usepackage{deepthink}
\hypersetup{
  citecolor=brown, 
}
\usepackage{mathtools} 
\usepackage{algorithm}
\usepackage{algpseudocode}
\usepackage{enumitem}
\usepackage{tcolorbox}
\usepackage{subcaption}
\usepackage{comment}
\usepackage{oubraces}
\usepackage{wrapfig}
\usepackage{cancel}

\usepackage[nameinlink]{cleveref}
\usepackage[
  backend=biber,
  style=alphabetic,
  maxbibnames=100,
  minbibnames=100,
  maxcitenames=2,
  mincitenames=2
]{biblatex}
\newcommand{\tr}{\mathrm{Tr}}

\def\mbf{\mathbf}
\def\mbb{\mathbb}
\def\mc{\mathcal}

\newcommand{\ui}{u_{\textup{ind}}}

\newcommand{\nin}{n_{\textup{ind}}}
\newcommand{\nov}{n_{\textup{over}}}

\newcommand{\lin}{\lambda_{\textup{ind}}}
\newcommand{\lov}{\lambda_{\textup{over}}}

\newtheorem{theorem}{Theorem}
\newtheorem{proposition}{Proposition}
\newtheorem{lemma}{Lemma}

\newtheorem{result}{Result}
\newtheorem{definition}{Definition}

\title{\LARGE{The Effect of Training Task Diversity on In-Context Learning through the Lens of Low-Dimensional Subspaces}}

\authorblock{
  \href{https://soominkwon.github.io/}{\textbf{Soo Min Kwon}}$^1$,
  \href{https://alecxu00.github.io/}{\textbf{Alec S. Xu}}$^1$,
  \href{https://canyaras.com/}{\textbf{Can Yaras}}$^1$,
\href{https://dogyoons.github.io/}{\textbf{Dogyoon Song}}$^2$,
  \href{https://web.eecs.umich.edu/~girasole/}{\textbf{Laura Balzano}}$^1$,
  \href{https://qingqu.engin.umich.edu/}{\textbf{Qing Qu}}$^1$
}

\affiliation{
  $^1$Department of Electrical and Computer Engineering, University of Michigan \\
  $^2$Department of Statistics, University of California, Davis
}

\abstracttext{
\noindent 
The transformer's emergent ability to perform in-context learning (ICL) has sparked a wide range of studies designed to understand its underlying mechanisms. Existing works often study how training task diversity, defined either as the number of ICL training task vectors or as the number of function classes from which the task vectors are drawn, shapes both the learning dynamics and generalization capabilities of ICL. While both definitions have uncovered many interesting phenomena, many observations under the latter definition remain theoretically unexplained. This paper presents a minimal analytical model under which these phenomena provably emerge from the properties of the training data. By modeling the training task vectors as a mixture of low-rank Gaussians, we show how training task diversity, defined by the number of non-overlapping columns between subspaces that parameterize the covariance matrices, improves both the generalization and optimization trajectory of ICL with linear attention. In particular, we show that our model can explain (i) why training with task diversity shortens the ICL plateau and (ii) why ICL appears to achieve out-of-distribution generalization. We conclude by empirically demonstrating how our results extend to nonlinear transformers and nonlinear function classes. Overall, our work presents a  tractable framework to unify existing observations.
}


\date{\today}
\correspondence{\href{kwonsm@umich.edu}{kwonsm@umich.edu}}
\resources{\href{https://github.com/soominkwon/ood-icl-generalize}{https://github.com/soominkwon/ood-icl-generalize}}

\headerlogo{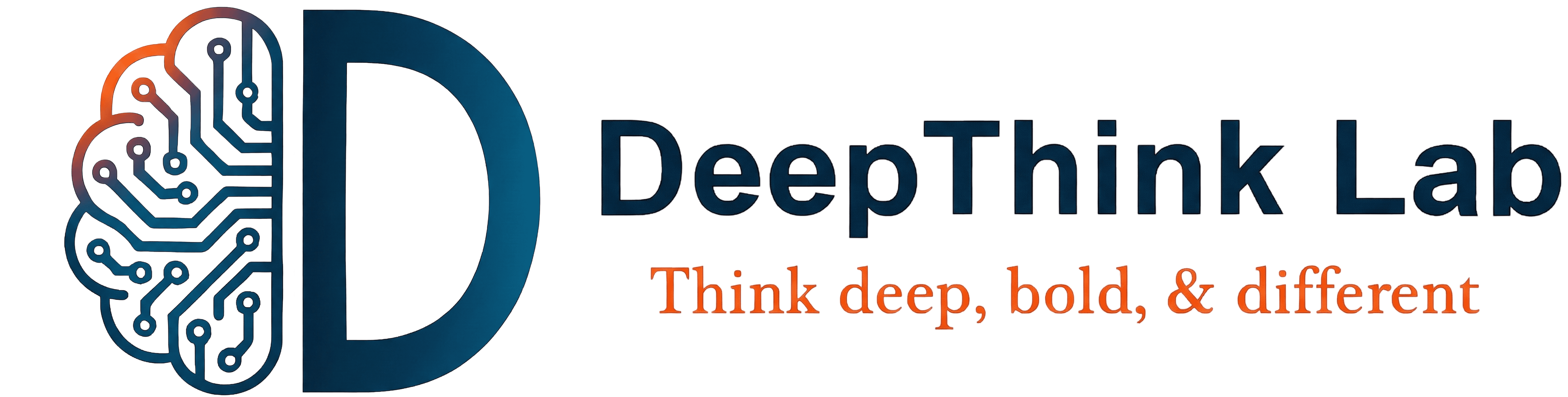}{https://deepthink-umich.github.io}

\begin{document}

\makeDeepthinkHeader
\vspace{-0.285in}
\begin{figure}[h]
    \centering
    \includegraphics[width=0.85\linewidth]{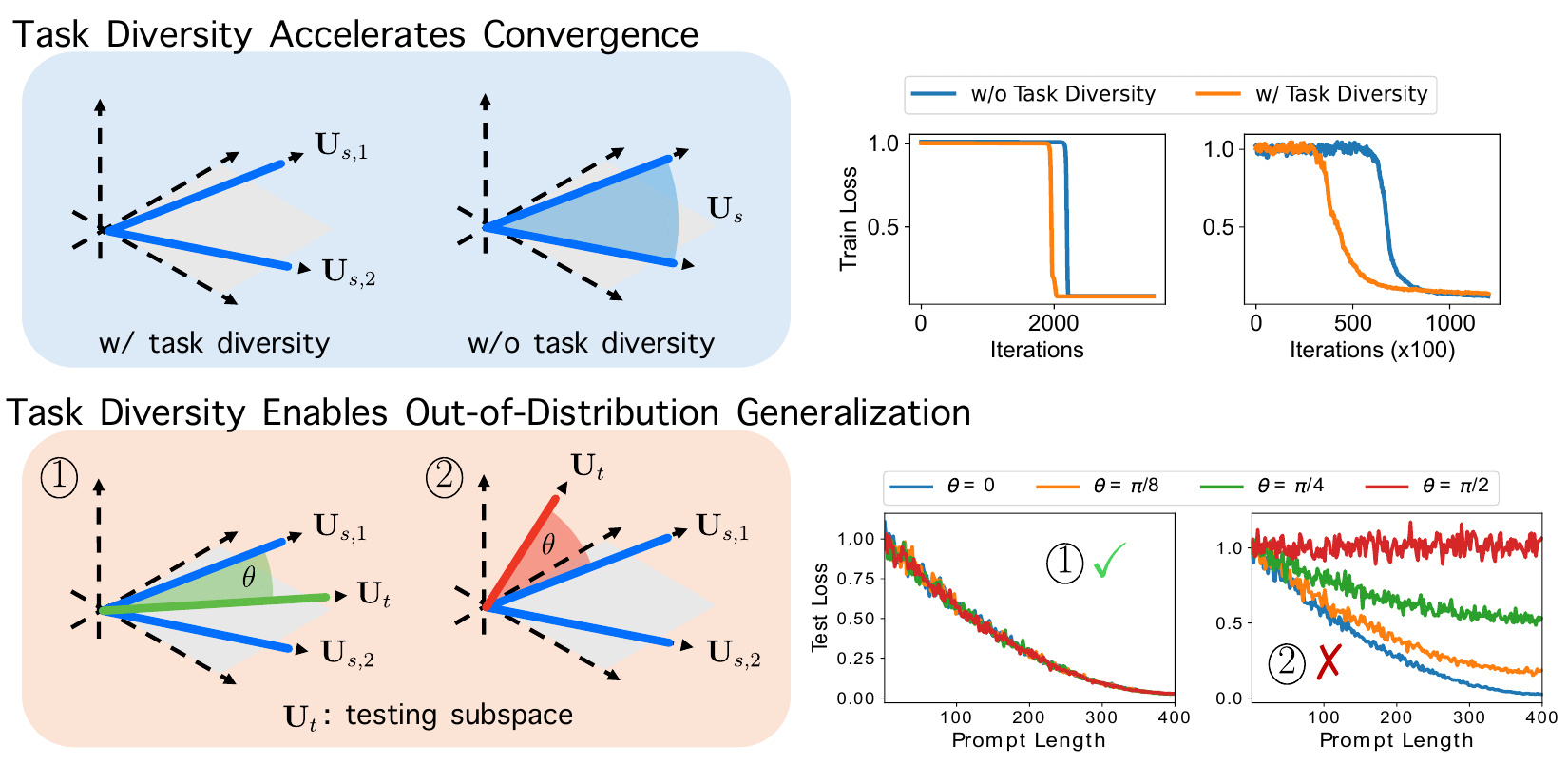}
    \caption{\small{\textbf{An illustrative overview.} We introduce a subspace-based notion of task diversity and prove its benefits for transformer learning dynamics and generalization. \textbf{Top:} task diversity accelerates convergence in both linear attention and GPT-2. \textbf{Bottom:} at the global minimum, a transformer trained with diverse tasks can generalize to all subspaces within the span of training subspaces at principal angle $\theta$, even in zero-density regions, but cannot generalize beyond this span.}}
    
    \label{fig:main_fig}
\end{figure}

\newpage
\tableofcontents


\newpage

\section{Introduction}
\label{sec:intro}

Transformer-based large language models (LLMs) \cite{vaswani2017attention} have revolutionized natural language processing and driven significant progress across a wide range of domains, including logical reasoning~\cite{wei2022chain}, sentiment classification~\cite{chen2024retrieval, wang2024chatgpt, xu2024improving}, machine translation~\cite{vilar2022prompting, agrawal2023context}, and code generation~\cite{li2023large, patel2024evaluating}.
Their success is largely attributed to scaling up model size, which has been shown to improve both performance and sample efficiency~\cite{kaplan2020scalinglawsneurallanguage}. Interestingly, large-scale transformers also exhibit emergent capabilities—abilities that arise only beyond a certain scale~\cite{wei2022emergent}. One striking example is in-context learning (ICL), where a model can perform a task simply by being prompted with a few input–output examples, without any gradient-based updates. 
This sparked both theoretical~\cite{zhang2024trained, huang2024context, li2024nonlinear, akyurek2023learning, von2023transformers, ahn2023transformers, li2024fine} and empirical~\cite{garg2022what, raventos2023pretraining, yadlowsky2023pretraining, wang2025can, ahuja2023closer, zhang2024trained, li2024nonlinear, pan2023context, kossen2024incontext} research aimed at understanding its strengths, limitations, and underlying mechanisms.

To better understand the capabilities of ICL, a growing body of work has focused on the effect of  training data and how it shapes ICL. Specifically, these works aim to identify how task diversity affects ICL, where a ``task'' refers to the function that generates the output corresponding to a given input.
Then, there are actually two different definitions of  training task diversity used in the literature: (i) the number of  training tasks in the training data, where each task is drawn from the \emph{same} function class, and (ii) the number of \emph{different} function classes represented in the training data. The first definition is often used to study the transition from memorization to generalization of ICL, and is better theoretically understood~\cite{raventos2023pretraining, asymptotic_theory}. The second definition has been used to explain a variety of other observations, including its effect on ICL's training dynamics~\cite{kim2025taskdiversityshortensicl} as well as ICL's out-of-distribution (OOD) capabilities~\cite{goddard2025when}.
However, many of these observations under the second definition remain theoretically underexplored, primarily because this notion of task diversity varies across works and is often defined in a way that is not amenable to analysis.

For example, when training transformers to perform ICL, it is often observed that the training loss remains plateaued for a substantial number of iterations before abruptly decreasing to near-optimal values~\cite{chen2024sudden, pulkit_abrupt, kim2025taskdiversityshortensicl}. Interestingly,~\cite{kim2025taskdiversityshortensicl} empirically showed that  training task diversity shortens this plateau, demonstrating that sampling task functions from different classes (e.g., linear regression, quadratic regression, and LeakyReLU regression) during training accelerates learning over training each function class independently. We reproduce this observation in \Cref{fig:icl-short-plateau} (left), where we show that a model trained with task diversity converges to a solution much faster than when training on each task individually.
Their hypothesis is that this plateau arises from the difficulty of learning a common structure across tasks, and that jointly training multiple tasks with shared structure accelerates learning. This in turn supports the use of diverse yet similar task functions in  training data for more efficient transformer training. While~\cite{kim2025taskdiversityshortensicl} provide strong empirical evidence for this hypothesis, a rigorous theoretical justification remains lacking.


Moreover,~\cite{goddard2025when} propose another definition of task diversity with an aim to demystify the OOD  capabilities of ICL.
In the literature, there exist mixed conclusions on whether ICL can generalize OOD: while many works claim that ICL is robust to distribution shifts in a linear regression setting~\cite{garg2022what, zhang2024trained}, \cite{wang2025can} recently challenged these views, empirically demonstrating that ICL generally succeeds only on in-distribution language data.
To unify these views,~\cite{goddard2025when} consider a setting in which  training task vectors are drawn from a subset of the unit hypersphere and define task diversity as the size of this subset. They show that increasing this size (and thus greater task diversity) is key to enabling OOD generalization. 
Unfortunately, their definition of the task vector made theoretical analysis intractable, restricting their study to empirical results.

While the works of~\cite{kim2025taskdiversityshortensicl} and~\cite{goddard2025when} highlight and advocate for task diversity in the  training data for training transformers, there is no clear theoretical justification for when and how task diversity specifically helps training transformers. In this work, we propose a mathematical framework that aims to unify these works and clearly identify how task diversity benefits ICL. We analyze the training dynamics and generalization behavior of ICL in a single-layer linear attention model for linear regression, where the  training task vectors are drawn from a mixture of Gaussians. Specifically, each mixture component is parameterized using a low-dimensional subspace, where subspaces across components share a common subspace. Then, we use this parameterization along with overlapping subspaces to define a notion of task diversity in a more principled way, while still reproducing a shortened plateau as observed by~\cite{kim2025taskdiversityshortensicl} and enabling OOD generalization similar to that of~\cite{goddard2025when} (see \Cref{fig:main_fig} for an illustrative overview).
Overall, our findings can be summarized as follows:

\begin{tcolorbox}[title=Summary of Contributions:]
\begin{itemize}[leftmargin=*]
    
    \item \textbf{Task Diversity Accelerates Convergence (\Cref{sec:acceleration}):} Using gradient flow, we show that having diverse task vectors yields a two-phase learning dynamic: an initial phase in which overlapping components across subspaces are learned, followed by a rapid learning phase of non-overlapping subspace components.  We derive the discrete-time convergence to a $\delta$-stationary point with an ansatz, and show that the convergence time largely scales with the number of shared components. This implies that the ICL plateau corresponds to learning this shared structure; once learned, individual components converge quickly, confirming the hypothesis of \cite{kim2025taskdiversityshortensicl} in our mathematical setting.
 
    \item \textbf{Task Diversity Enables Out-of-Distribution Generalization (\Cref{sec:ood}):} Our mathematical model allows us to define subspaces at specific principal angles relative to the  training subspaces. We show that a transformer trained on tasks drawn from a mixture of Gaussians can generalize to subspaces across all principal angles within the training span, including regions with zero probability density under the training distribution. In contrast, when tested on subspaces outside the training span, the model incurs test risk as a function of the angle, indicating it cannot generalize beyond the training subspaces.
Together, these results provide a principled explanation for the OOD capabilities of ICL observed in the literature: test task vectors must lie within the span of the training task vectors.
\end{itemize}
\end{tcolorbox}

\paragraph{Notation.}
We denote scalars with unbolded letters (e.g., $m, M$), vectors with bold lower-case letters (e.g., $\mbf{x}$)
and matrices with bold upper-case letters (e.g., $\mbf{X}$).
We use $\mbf{I}_n$ to denote an identity matrix of size $n \in \mbb{N}$. We use $\mc{R}(\mbf{X})$ to denote the range or the column space of the matrix $\mbf{X}$. Lastly, given any $n \in \mbb{N}$, we use $[n]$ to denote the index set $\{1, \ldots, n\}$. 

\section{Background and Problem Setup} \label{sec:setup}

To facilitate analysis, we require a tractable model and experimental setup that exhibit ICL. \cite{ahn2024linear} empirically demonstrated that many phenomena observed in vanilla transformers can also be replicated in transformers with linear attention. These findings motivated subsequent works~\cite{ahn2023transformers, zhang2024trained, li2024fine, wu2024how, gozeten2025testtime, chang2025taskspecific} to adopt linear attention together with linear regression as a testbed for studying ICL, an approach we also follow in our analysis.
In the following, we introduce these two components in detail, discuss their relationship to ICL, and explain how we use linear regression to define a new notion of task diversity.

\subsection{In-Context Learning Setup}

Given a sequence of $n$ input-output example pairs $\{\mbf{x}_i, y_i\}_{i=1}^n \subset \mbb{R}^d \times \mbb{R}$, the objective of ICL is to predict the output $y_{n+1} \in \mbb{R}$ corresponding to an unseen query $\mbf{x}_{n+1} \in \mbb{R}^d$. Following prior works  \cite{garg2022what}, we assume each output is generated via $y_i = f(\mbf{x}_i)$ for some function $f:\mbb{R}^d \to \mbb{R}$, where $f \in \mc{F}$ is sampled from a distribution over a function class $\mc{F}$. By convention, a transformer takes in these pairs as a prompt $\mbf{Z} \in \mbb{R}^{(n+1) \times (d+1)}$ constructed in the following form: 
\begin{align*}
    \mbf{Z} &= \begin{bmatrix}
    \mbf{z}_1 & \ldots & \mbf{z}_n & \mbf{z}_{n+1}
    \end{bmatrix}^\top = \begin{bmatrix}
    \mbf{x}_1 & \ldots & \mbf{x}_n & \mbf{x}_{n+1} \\
    y_1 & \ldots & y_n & 0
    \end{bmatrix}^\top,
\end{align*}
where $\mbf{z}_i := \begin{bmatrix}
    \mbf x_i^\top & y_i
\end{bmatrix}^\top$ and $\mbf{z}_{n + 1} := \begin{bmatrix}
    \mbf x_{n + 1}^\top & 0
\end{bmatrix}^\top$. 
Then, a transformer $g_{\mc{W}} : \mbb{R}^{(n+1)\times (d+1)} \to \mbb{R}$, parameterized by weights $\mc{W}$, takes these prompts as input and is trained by minimizing the following expected squared loss with respect to $\mc{W}$:
\begin{align}
\label{eqn:expected_lin_att_objective}
\underset{\mc{W}}{\min} \,\, \mc{L}(\mc{W}) := 
\mbb{E}\left[\left(y_{n+1} - g_{\mc{W}}(\mbf{Z})  \right)^2 \right].
\end{align}
During inference time, we test the trained  model denoted $g^\star_{\mc{W}}$ using $m + 1$ paired examples $\{\mbf{x}_j, \tilde{y}_j\}_{j=1}^{m+1}$. The input prompts are constructed in the same manner:
\begin{align*}
&\widetilde{\mbf{Z}} = \begin{bmatrix}
    \mbf{x}_1 & \ldots & \mbf{x}_m & \mbf{x}_{m+1} \\
    \widetilde{y}_1 & \ldots & \widetilde{y}_m & 0
    \end{bmatrix}^\top 
    \,\, \text{and} \quad \widetilde{\mbf{z}}_{m+1} = \begin{bmatrix}
         \mbf{x}_{m+1} \\ 0
    \end{bmatrix},
\end{align*}
where the labels are generated via $\widetilde{y}_j = \widetilde{f}(\mbf{x}_j)$ for some function $\widetilde{f} \in \mc{F}$.

\subsection{Single-Layer Linear Attention}

For linear attention, given the full prompt $\mbf{Z}$, we consider the following masked prompt~\cite{ahn2023transformers, mahankali2024one}:
\begin{align*}
\mbf{Z}_{\mc{M}} = 
    \begin{bmatrix}
    \mbf{z}_1 & \ldots & \mbf{z}_n & \mbf{0}
    \end{bmatrix}^\top.
\end{align*}
Then, the single-layer linear attention model sets $g_{\mc{W}}$ as follows to make the prediction $\widehat{y}_{n+1}$: 
\begin{align}
\label{eqn:linear_att} 
\widehat{y}_{n+1} &= g_{\mc{W}}(\mbf{Z}) = \frac{1}{n}\left(\mbf{z}_{n+1}^\top\mbf{W}_{Q} \mbf{W}_K^\top \mbf{Z}_\mc{M}^\top \right) \mbf{Z}_\mc{M} \mbf{W}_V \mbf{p},
\end{align}
where $\mbf{p} =[ \mbf{0}_d \,\,\, 1]^\top$
and $\mbf{W}_K, \mbf{W}_{Q}, \mbf{W}_V \in \mbb{R}^{(d+1) \times (d+1)}$ are the key, query, and value weight matrices, respectively. This sets $\mc{W} = \{\mbf{W}_K, \mbf{W}_{Q}, \mbf{W}_V\}$ as the collection of trainable weights corresponding to the linear attention model. 
Then, let $\mc{W}^\star = \{\mbf{W}_K^\star, \mbf{W}_Q^\star, \mbf{W}_V^\star\}$ be the optimal weights obtained by minimizing the loss in \Cref{eqn:expected_lin_att_objective}. During inference time, we test the optimal linear attention model denoted as $g^\star_{\mc{W}}$ using the $m + 1$ paired examples:
\begin{align*}
   \widehat{y}_{m+1} &=  g^\star_{\mc{W}}\left( \widetilde{\mbf{Z}} \right) = \frac{1}{m}\left(\widetilde{\mbf{z}}_{m+1}^\top\mbf{W}_Q^\star \mbf{W}_K^{\star\top} \widetilde{\mbf{Z}}_\mc{M}^\top \right) \widetilde{\mbf{Z}}_\mc{M} \mbf{W}_V^\star \mbf{p},
\end{align*}
where we normalize by a factor of $m$ instead of $n$.

\begin{figure}[t!]
  \centering

  \begin{subfigure}[t]{0.45\linewidth}
    \centering
    \includegraphics[width=\linewidth]{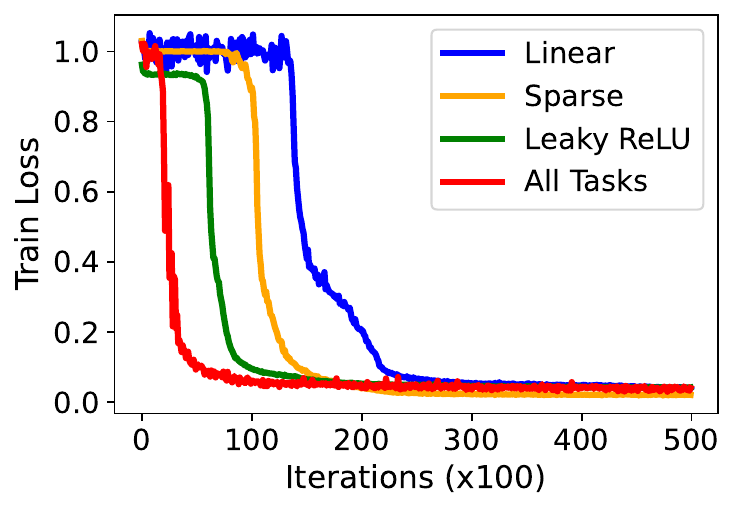}
  \end{subfigure}
  \hfill
  \begin{subfigure}[t]{0.45\linewidth}
    \centering
    \includegraphics[width=\linewidth]{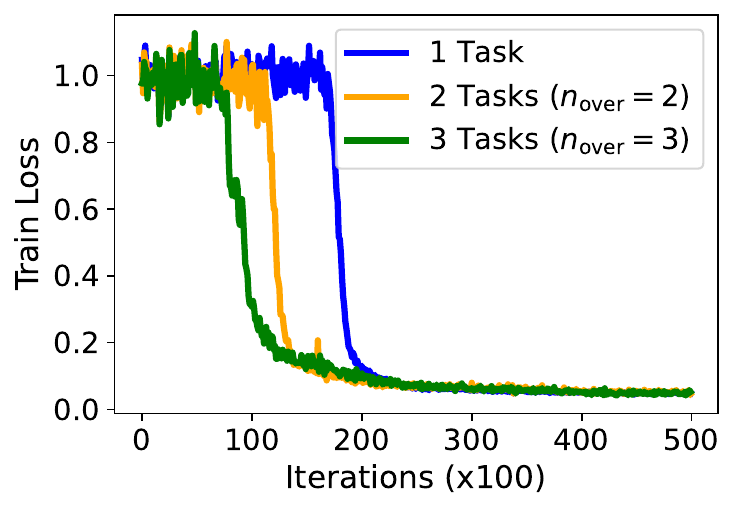}
  \end{subfigure}

  \caption{\textbf{Demonstrating the effects of training task diversity.} We train a GPT-2 model for ICL, and show that task diversity shortens the ICL plateau. \textbf{Left:} Reproducing the observations of~\cite{kim2025taskdiversityshortensicl} with $d=10$; training on all tasks jointly (linear, sparse, Leaky ReLU) drops the training loss faster than that of training on each task individually. \textbf{Right:} Plot of the training loss using our definition of task diversity with $d=15$ and $q=12$. When drawn from multiple tasks with shared directions across all task vectors, the ICL plateau drops more quickly than for a single individual task. In both figures, we draw from each task (or function class) with equal probability.}
  \label{fig:icl-short-plateau}
\end{figure}

\subsection{Linear Regression and Training Task Diversity}
\label{sec:diversity}

With linear attention, the most commonly studied ICL setup is the linear regression setting, which specifies $f(\mbf{x}) = \mbf{w}^\top \mbf{x}$ for some task vector $\mbf{w} \in \mbb{R}^d$.
In this section, our goal is to use linear regression to define a  notion of task diversity similar to \cite{goddard2025when, kim2025taskdiversityshortensicl}, but in a more tractable manner.

Recall that~\cite{kim2025taskdiversityshortensicl} defines training task diversity as drawing from multiple different task functions during training (e.g., linear regression, Leaky ReLU regression), sampled according to some probability distribution, rather than from a single task function. Following this definition, suppose there are a total of $K$ different \underline{s}ource tasks we wish to learn. Instead of using different function classes, we capture task diversity through $K$ Gaussian distributions that differ only in their covariance structure:
\begin{align*}
    \text{Task 1: } \mathcal{N}(\mathbf{0}, \mathbf{\Sigma}_{s, 1}),
    \quad \text{Task 2: } \mathcal{N}(\mathbf{0}, \mathbf{\Sigma}_{s, 2}),
    \quad \ldots, \quad
    \text{Task $K$: } \mathcal{N}(\mathbf{0}, \mathbf{\Sigma}_{s, K}),
\end{align*}
where $\mbf{\Sigma}_{s,k} \in \mbb{R}^{d\times d}$. We parameterize each covariance matrix using a low-rank orthonormal basis $\mbf{U}_{s,k} \in \mathbb{R}^{d \times r_k}$. We require that these task bases share a common subspace:
\begin{align*}
    \nov \coloneqq \mathrm{dim}\left(\,\,\bigcap_{k=1}^K \mathrm{span}\left(\mbf{U}_{s, k} \right)\right) > 0, \quad \text{with} \quad \nov < q,
\end{align*}
where we assume that any two task subspaces intersect only in this common subspace and are otherwise mutually orthogonal. Then, each covariance matrix takes the form:
\begin{align*}
    \mbf{\Sigma}_{s, k} = \mbf{U}_{s, k}\mbf{U}_{s, k}^\top + \epsilon \cdot \mbf{I}_d,
\end{align*}
where $\epsilon > 0$ is a small constant to ensure a non-degenerate distribution.

The intuition for this setup is as follows:~\cite{kim2025taskdiversityshortensicl} hypothesize that task diversity shortens the ICL plateau because jointly training on multiple tasks with shared structure makes that structure easier to learn. This setup allows us to directly test this hypothesis by checking whether having $0 < \nov < q$ accelerates convergence, while the mutual orthogonality constraint ensures that each task contains independent components to be learned alongside the shared structure. Then, we say that a training task vector is \emph{diverse} if it is drawn from a mixture of at least two distributions whose subspaces have a non-trivial intersection. 

We compare this to the non-diverse (or single-task) setting, which draws from a single orthonormal basis $\mbf{U}_s \in \mbb{R}^{d\times q}$ that spans the union of all task subspaces, $\mc{U}_s$, defined as:
\begin{align*}
    \mc{U}_s \coloneqq \mathrm{span}\left(\mbf{U}_{s, 1} \right) + \ldots + \mathrm{span}\left(\mbf{U}_{s, K} \right), \quad \text{where} \quad \mathrm{dim}(\mc{U}_s) = q < d.
\end{align*}
While both setups span the same overall subspace $\mc{U}_s$ and aim to learn all $K$ tasks, this allows us to study task diversity can accelerate learning, as shown in \Cref{fig:icl-short-plateau}. Below, we formalize the data generation process for both task diverse and non-diverse setups.

\begin{tcolorbox}
\begin{definition}[Task Diversity]
\label{def:training_tasks}
The training prompts consist of  pairs $\{(\mbf{x}_i, y_i)\}_{i=1}^{n+1}$, where for each $i \in [n+1]$, the inputs are drawn independently as $\mbf{x}_i \sim \mc{N}(\mbf{0}, \mbf{I}_d)$ and the targets are generated by
\begin{equation}
\label{eq:task_setup}
y_i = \mbf{w}^\top \mbf{x}_i + \xi_i, \quad \text{where} \quad \xi_i \sim \mc{N}(0, \sigma^2),
\end{equation}
and $\sigma \geq 0$ represents the noise level. 
We define task diversity based on the distribution of the task vector $\mbf{w}$:
\begin{enumerate}[label=\textup{(\roman*)}]
    \item \textbf{With Diversity:} $\mbf{w}$ is drawn from a mixture of $K \geq 2$ Gaussians: 
    \begin{equation}
    \label{eq:mog_setup}
    \mbf{w} \sim \frac{1}{K} \sum_{k=1}^K \mc{N}(\mbf{0}, \mbf{\Sigma}_{s, k}) = \begin{cases}
        \mathcal{N}(\mathbf{0}, \mathbf{\Sigma}_{s, 1}) &\text{w.p.  }  1/K, \\
        \quad\quad\vdots \\
        \mathcal{N}(\mathbf{0}, \mathbf{\Sigma}_{s, K}) &\text{w.p.  } 1/K.
    \end{cases}
    \end{equation}
    \item
    \textbf{Without Diversity:} $\mbf{w}$ is drawn from a single Gaussian, $\mbf{w} \sim \mc{N}(\mbf{0}, \mbf{\Sigma}_s)$:
    \begin{equation}
    \label{eq:single_cov}
    \mbf{\Sigma}_s = \mbf{U}_s\mbf{U}_s^\top +\epsilon \cdot \mbf{I}_d, \quad \text{where} \quad \mbf{U}_s \in \mbb{R}^{d\times q} \text{  is an orthonormal basis for $\mc{U}_s$}.
    \end{equation}
\end{enumerate}
\end{definition}
\end{tcolorbox}
\noindent We remark that this formulation for the task vectors employs a union-of-subspaces model, which has recently served as a theoretical testbed across a wide range of domains~\cite{wang2024diffusion,xu2025understanding}.
\section{Task Diversity Accelerates Convergence}
\label{sec:acceleration}

In this section, we study how task diversity, as defined in \Cref{sec:diversity}, benefits transformer learning dynamics. This section is organized as follows: \Cref{sec:gf_dynamics} establishes the general gradient flow (GF) dynamics of a single-layer linear attention model; \Cref{sec:main_speedup} states our main result quantifying the convergence speedup from task diversity; and \Cref{sec:single_subspace} and \Cref{sec:two_stage_paradigm} develop the proof sketch for the single-task and task-diverse cases, respectively.

To simplify the analysis, we consider the $K=2$ case of
\Cref{def:training_tasks} where $r = r_1 = r_2$: for $i \in [n+1]$, we draw
$\mbf{x}_i \sim \mc{N}(\mbf{0}, \mbf{I}_d)$ independently and
\begin{align}
\label{eq:two_setup}
y_i = \mbf{w}^\top \mbf{x}_i + \xi_i,
\quad \xi_i \sim \mc{N}(0, \sigma^2),
\quad \mbf{w} \sim 
\begin{cases}
\mc{N}(\mbf{0}, \mbf{\Sigma}_{s, 1}) & \text{w.p.\ } 1/2, \\
\mc{N}(\mbf{0}, \mbf{\Sigma}_{s, 2}) & \text{w.p.\ } 1/2.
\end{cases}
\end{align}
We compare this to the single-task case, which sets $\mc{U}_s = \mathrm{span}(\mbf{U}_{s, 1}) +  \mathrm{span}(\mbf{U}_{s, 2})$ with $\mathrm{dim}(\mc{U}_s) = q$, $\mathrm{rank}(\mbf{U}_{s, 1}) = \mathrm{rank}(\mbf{U}_{s, 2}) = r$, and $r > q/2$ such that $\nov > 0$. 
Furthermore, we will make the following assumptions throughout this section: (i) the component covariance matrices are normalized, i.e., $\tr(\mbf{\Sigma}_{s, 1}) = \tr(\mbf{\Sigma}_{s, 2}) = 1$ (which implies $\tr(\mbf{\Sigma}_s) = 1$), and (ii) the covariance matrices are exactly low-rank, i.e., $\epsilon = 0$.

\subsection{Gradient Flow Training Trajectory}
\label{sec:gf_dynamics}

Following the work of \cite{zhangtraining}, we first derive the general GF dynamics of linear attention with respect to the expected loss in \Cref{eqn:expected_lin_att_objective}, subject to the following initial conditions:
\begin{align}
\label{eqn:gf_initial}
        \mbf{W}_Q(0) = \mbf{W}_K(0) = \alpha \cdot \mbf{I}_{d+1} \quad \text{and} \quad \mbf{W}_V(0) = \phi \cdot \mbf{I}_{d+1},
\end{align}
where $\alpha > 0$ is a small constant and $\phi = \alpha\sqrt{q}$. 
The small initialization constant along with GF is often used in analyzing learning dynamics, as it has been shown to yield richer behavior compared to operating in the kernel regime~\cite{zhangtraining, jacot2022saddletosaddledynamicsdeeplinear, domine2025from, gissin2020the, stoger2021smallakin, kwon2024compression, li2021towards}. The value matrix is initialized using $\phi$ instead of $\alpha$ to satisfy a conservation law used in our analysis (see \Cref{lemma:conservation}). Then, the following result describes the limiting GF dynamics. 

\begin{tcolorbox}
    \begin{proposition}
\label{prop:gf_dynamics}
Suppose the task vector $\mbf{w} \in \mbb{R}^d$ is a zero-mean random vector with covariance matrix $\mbf{\Sigma}_s \coloneqq \mbb{E}[\mbf{w}\mbf{w}^\top]$, whose eigendecomposition is given by $\mbf{\Sigma}_s = \mbf{U}_s\mbf{\Lambda}_s \mbf{U}_s^\top$.
Given the initial conditions in \Cref{eqn:gf_initial}, the linear attention model in \Cref{eqn:linear_att} is equivalent to
\begin{align*}
     g_{\mc{W}}(\mbf{Z}) = v\mbf{x}_{n+1}^\top \mbf{Q}\mbf{K}^\top \mbf{c}, \quad \text{where} \quad \mbf{c} = \frac{1}{n} \sum_{i=1}^n y_i \mbf{x}_i,
\end{align*}
$v \in \mbb{R}$ is the $(d{+}1, d{+}1)$ entry of $\mbf{W}_V$, and $\mbf{Q}, \mbf{K} \in \mbb{R}^{d\times d}$ are the top-left principal submatrices of $\mbf{W}_{Q}$ and $\mbf{W}_K$, respectively.
Furthermore, they have the following limiting GF dynamics:
\begin{align*}
    \tau \underline{\dot{\mbf{Q}}} \coloneqq\tau \lim_{n\to \infty} \dot{\mbf{Q}} &= \tau \lim_{n\to \infty} \mbf{U}_s \dot{\mbf{\Lambda}}_Q \mbf{U}_s^\top =  \mbf{U}_s \left(v\left(\mbf{\Lambda}_s - v \mbf{\Lambda}_{Q} \mbf{\Lambda}_{K} \mbf{\Lambda}_s \right) \mbf{\Lambda}_{K} \right) \mbf{U}_s^\top,\\ 
    \tau\underline{\dot{\mbf{K}}} \coloneqq \tau \lim_{n\to \infty} \dot{\mbf{K}} &= \tau \lim_{n\to \infty} \mbf{U}_s \dot{\mbf{\Lambda}}_K \mbf{U}_s^\top =  \mbf{U}_s\left(v\mbf{\Lambda}_{Q} \left(\mbf{\Lambda}_s - v \mbf{\Lambda}_{Q} \mbf{\Lambda}_{K} \mbf{\Lambda}_s \right)\right) \mbf{U}_s^\top, \\
    \tau\underline{\dot{v}} \coloneqq \tau \lim_{n\to \infty}\dot{v} &= \tr\left(\mbf{\Lambda}_{Q} \left( \mbf{\Lambda}_s - v \mbf{\Lambda}_{Q} \mbf{\Lambda}_{K} \mbf{\Lambda}_s\right)\mbf{\Lambda}_{K}\right),
\end{align*}
where $\tau$ is a time constant, and
$\mbf{Q}(t)=\mbf{U}_s\mbf{\Lambda}_Q(t)\mbf{U}_s^\top$ and
$\mbf{K}(t)=\mbf{U}_s\mbf{\Lambda}_K(t)\mbf{U}_s^\top$, with
$\mbf{\Lambda}_Q(t)$ and $\mbf{\Lambda}_K(t)$ diagonal.
\end{proposition}
\end{tcolorbox}

The time constant $\tau$ is set to $\tau = 1 / \eta$ for some learning rate $\eta$, and can be viewed as a conversion factor between the discrete iteration index and the continuous-time derivative used for analysis. \Cref{prop:gf_dynamics} provides the dynamics for a general task vector whose mean is centered at zero and whose covariance matrix is $\mbf{\Sigma}_s \in \mbb{R}^{d \times d}$, under a weight simplification used in related work~\cite{zhangtraining, asymptotic_theory}. The initial identity condition in 
\Cref{eqn:gf_initial} simplifies the dynamics such that $\mbf{Q}, \mbf{K} \in \mbb{R}^{d\times d}$ are simultaneously diagonalizable in the eigenbasis of $\mbf{\Sigma}_s \in \mbb{R}^{d\times d}$, and remain static throughout learning. This allows us to reduce the limiting dynamics into $\mbf{\Lambda}_Q$ and $\mbf{\Lambda}_K$:
\begin{align}
\label{eqn:diag_dynamics}
    \tau \underline{\dot{\mbf{\Lambda}}}_Q = v\left(\mbf{\Lambda}_s - v \mbf{\Lambda}_{Q} \mbf{\Lambda}_{K} \mbf{\Lambda}_s \right) \mbf{\Lambda}_{K} \quad \text{and} \quad \tau \underline{\dot{\mbf{\Lambda}}}_K = v\mbf{\Lambda}_{Q} \left(\mbf{\Lambda}_s - v \mbf{\Lambda}_{Q} \mbf{\Lambda}_{K} \mbf{\Lambda}_s \right).
\end{align}
Intuitively, this implies that in the limit, the linear attention model simply learns to scale the magnitudes of the features along the fixed principal directions defined by $\mbf{U}_s$.
Then, since all components in \Cref{eqn:diag_dynamics} are diagonal, it suffices to track the dynamics of each coordinate separately. Moreover, because $\mbf{\Lambda}_{Q}(0) = \mbf{\Lambda}_{K}(0) = \alpha \cdot \mbf{I}_d$, they follow identical dynamics. Let $\lambda_i$ denote the $i$-th diagonal element of both $\mbf{\Lambda}_{Q}$ and $\mbf{\Lambda}_{K}$, and let $\sigma_i \geq 0$ denote the $i$-th diagonal element of $\mbf{\Lambda}_s$. Then, we have the following:
\begin{align}
\label{eqn:final_gf}
    \tau \underline{\dot{v}} = \sum_{i=1}^d \left(\sigma_i - v\sigma_i\lambda_i^2 \right)\cdot \lambda_i^2 \quad \text{and} \quad 
    \tau \underline{\dot{\lambda}}_i = \left(\sigma_i - v \sigma_i\lambda_i^2 \right) \cdot v  \lambda_i, \quad \forall i \in [d]. 
\end{align}
The dynamics are now determined solely by the initial conditions on the weights and the signal strengths $\sigma_i \geq 0$. This is the primary machinery we will need for this section: the signal strengths $\sigma_i$ are the only quantities in \Cref{eqn:final_gf} that differ between the single-task and task-diverse settings, so any difference in convergence time must trace back to them. In the following subsections, we plug in $\sigma_i$ for each case and derive the number of discrete steps required to reach a $\delta$-stationary point. Due to the nature of gradient flow, the dynamics in \Cref{eqn:final_gf} admit a conservation law that will be useful in the analysis:
\begin{tcolorbox}
    \begin{lemma}
\label{lemma:conservation}
With the initial conditions in \Cref{eqn:gf_initial}, the gradient flow dynamics in \Cref{eqn:final_gf} admit the following conservation law:
\begin{align*}
\sum_{i=1}^q \lambda_i^2(t) - v^2(t) = \sum_{i=1}^q \lambda_i^2(0) - v^2(0) = 0.
\end{align*}
\end{lemma}

\end{tcolorbox}

\subsection{Main Result: Convergence-Time Decomposition}
\label{sec:main_speedup}

With the GF dynamics in place, we are in a position to present our main result for this section, which establishes the discrete iteration time required to reach a $\delta$-stationary point in both settings and shows that task diversity  accelerates convergence. The following result shows that both settings have a \emph{growth} time, which denotes the time spent escaping the small-initialization regime, and a \emph{saturation} time, which denotes the time spent saturating near the stationary point. However, the task-diverse setting further decomposes the growth time into a stage in which the overlapping components are learned, followed by a stage in which the individual components are learned, yielding a two-stage learning phenomenon.

\begin{result}[Convergence Time Decomposition]
\label{res:time_decomp}
Suppose the linear attention model follows the GF dynamics in \Cref{prop:gf_dynamics} with initialization scale $\alpha > 0$ in \Cref{eqn:gf_initial} and fix a stationarity tolerance $\delta \in (0, q^{-1/6})$. Then the discrete convergence times are as follows:

\noindent\textup{(i)} \textbf{Single-Task Setting:} The convergence time to reach a $\delta$-stationary point satisfies
\begin{align}
\label{eqn:main_sing_decomp}
    t_{\textup{single}}
    =
    \underbrace{\tau \sqrt{q}
    \left(
        \frac{1}{\alpha}
        -
        \frac{1}{q^{-1/6}-\delta}
    \right)}_{\eqqcolon \, t_{\textup{single,growth}}}
    +
    \underbrace{\tau q^{2/3}
    \left(
        h(1-q^{1/6}\delta)
        -
        h(q^{1/6}\alpha)
    \right)}_{\eqqcolon \, t_{\textup{single,sat}}},
\end{align}
where $t_{\textup{single,growth}}$ denotes the time spent escaping the small-initialization regime, $t_{\textup{single,sat}}$ denotes the time spent saturating near the stationary point, and
\begin{align}
\label{eqn:h_x_thm}
     h(x) \coloneqq -\frac{1}{3} \log |x - 1| + \frac{1}{6} \log(x^2 + x + 1) - \frac{\sqrt{3}}{3} \arctan\!\left(\frac{2x + 1}{\sqrt{3}}\right).
\end{align}

\noindent \textup{(ii)} \textbf{Task-Diverse Setting:} Using an ansatz on the GF dynamics, the convergence time $t_{\textup{mix}}$ to reach a $\delta$-stationary point decomposes as $t_{\textup{mix}} = t_{\textup{over}} + t_{\textup{ind}}$, where $t_{\textup{over}}$ is the time during which the shared, overlapping components are learned first, and $t_{\textup{ind}}$ is the time during which the individual components are learned thereafter:
\begin{enumerate}[label=\textup{(\alph*)}]
    \item The shared-stage time $t_{\textup{over}}$ further decomposes into growth and saturation contributions:
    \begin{align}
    \label{eqn:main_mix_decomp}
        t_{\textup{over}} = \underbrace{\frac{2\tau a}{\alpha \nin}
    \left(
        \sqrt{q}
        -
        \sqrt{\nov + \frac{\nin \alpha}{\lov^\star-\delta}}
    \right)}_{\eqqcolon \, t_{\textup{over,growth}}}
    +
    \underbrace{a\tau \left( \lov^{\star} \right)^2\cdot
    \left(
        h\!\left(\frac{\lov^\star-\delta}{\lov^\star}\right)
        -
        h\!\left(\frac{\alpha}{\lov^\star}\right)
    \right)}_{\eqqcolon \, t_{\textup{over,sat}}},
    \end{align}
    where $a \coloneqq (2\nov + \nin)/2$ is a normalization constant, $h(\cdot)$ is defined in \Cref{eqn:h_x_thm}, and $\lov^\star$ is the unique positive root of
    \begin{align}
    \label{eqn:main_lambda_star_eqn}
        \nov \left(\lov^\star\right)^6
        +
        \alpha\nin\left(\lov^\star\right)^5
        -
        1
        =
        0.
    \end{align}
    \item The individual-stage time $t_{\textup{ind}}$ satisfies the bounds
    \begin{align}
    \label{eqn:main_ind_bound}
        \tau L \cdot g(1-\delta)
        \;\leq\;
        t_{\textup{ind}}
        \;\leq\;
        \tau L \cdot g(\alpha/\lov^\star),
    \end{align}
    where we define
    \begin{align*}
        L \coloneqq a \log\!\left(
            \frac{(1-\delta)(1-\alpha/\lov^\star)}{\delta\alpha/\lov^\star}
        \right)
        \quad \text{and} \quad
        g(u) \coloneqq \frac{(\nov + \nin u)^{2/3}}{\nov + 2\nin u}, \quad u \in (0,1).
    \end{align*}
    
\end{enumerate}
\end{result}
We provide proof sketches for $t_{\text{single}}$ and $t_{\text{mix}}$ in \Cref{sec:single_subspace} and \Cref{sec:two_stage_paradigm}, respectively. Compared to the single-task setting, we show in \Cref{sec:two_stage_paradigm} that the task-diverse case yields complicated, coupled GF dynamics. To handle this coupling, we introduce an ansatz: by assuming the
individual components follow a power-law relationship with the shared
components until the shared components reach stationarity, we can
reduce the dynamics to a 1D system. Furthermore, this coupled nature means we can only derive upper and lower bounds for $t_{\text{ind}}$, which consequently bound the total convergence time $t_{\text{mix}}$.
We find empirically that the lower bound provides a tight approximation of $t_{\text{ind}}$ (see \Cref{fig:power_law_pred_time}), but more importantly, $t_{\text{ind}}$ is itself negligible compared to $t_{\text{over}}$: both bounds in \Cref{eqn:main_ind_bound} scale as $\tau a \cdot \log(1/\alpha) \cdot g(u)$ for $u \in (0,1)$, where $g$ is decreasing on this interval with
\begin{align*}
    \sup_{u \in (0,1)} g(u)
    = \lim_{u \to 0^+} g(u)
    = \frac{\nov^{2/3}}{\nov}
    = \nov^{-1/3}.
\end{align*}
Hence $t_{\text{ind}} = \mc{O}(\log(1/\alpha))$, whereas $t_{\text{over}} = \Theta(1/\alpha)$, so the primary bottleneck in $t_{\text{mix}}$ is the growth time of $t_{\text{over}}$. 
We further show in the derivation of \Cref{res:strict_acceleration} that the saturation times are likewise subleading,
and so it suffices to compare $t_{\textup{single,growth}}$ to $t_{\textup{over,growth}}$ to establish the speedup. This leads to the following result:

\begin{result}[Strict Acceleration Approximation via Task Diversity]
\label{res:strict_acceleration}
Define the overlap ratio
\begin{align}
\label{eqn:kappa_def}
    \kappa \;\coloneqq\; \frac{\nov}{q} \;\in\; [0, 1],
\end{align}
which measures the fraction of the training subspace shared by the two tasks: $\kappa = 0$ corresponds to disjoint subspaces, while $\kappa = 1$ corresponds to a single shared subspace. For sufficiently small $\alpha$, the leading growth times satisfy
\begin{align}
\label{eqn:main_ratio}
    \frac{t_{\textup{over,growth}}}{t_{\textup{single,growth}}}
    \;\approx\;
    \rho(\kappa)
    \;\coloneqq\;
    \frac{1 + \kappa}{1 + \sqrt{\kappa}},
\end{align}
where $\rho$ has the following properties:
\begin{enumerate}[label=\textup{(\roman*)}]
    \item \textbf{Strict Acceleration:} $\rho(\kappa) < 1$ for every $\kappa \in (0, 1)$.
    \item \textbf{Optimal Diversity:} $\rho$ attains its minimum at $\kappa^\star = 3 - 2\sqrt{2}$, with $\rho(\kappa^\star) = 2(\sqrt{2} - 1)$, corresponding to roughly a $17\%$ reduction in the leading growth time.
\end{enumerate}
\end{result}

\begin{figure}[t!]
  \centering

  \begin{subfigure}[t]{0.48\linewidth}
    \centering
    \includegraphics[width=\linewidth]{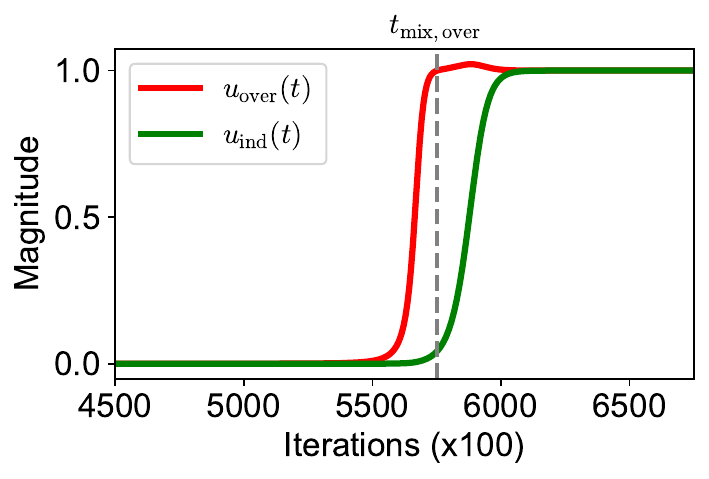}
    \label{fig:pred_time_a}
  \end{subfigure}
  \hfill
  \begin{subfigure}[t]{0.48\linewidth}
    \centering
    \includegraphics[width=\linewidth]{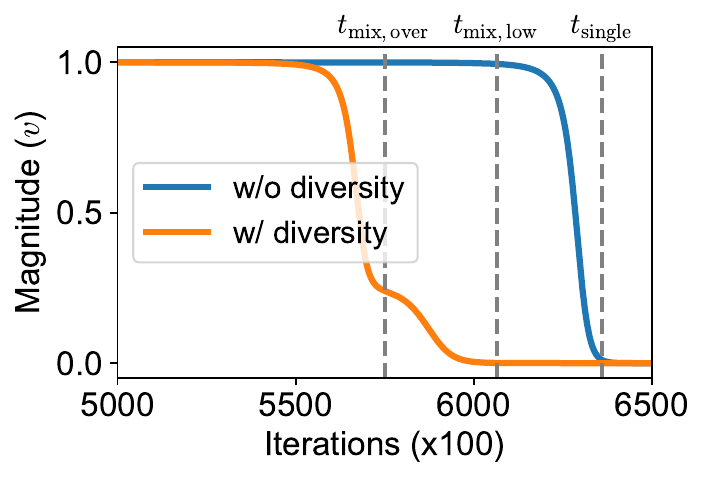}
  \end{subfigure}

  \caption{\textbf{Plots for simulated GF dynamics.} We choose with $q=10$, $\nin = 4$, $\nov = 6$, $\alpha = 0.005$, $\eta = 0.001$, and $\delta = 0.001$. $t_{\text{mix, over}}$ denotes the predicted value of $t_{\text{over}}$ for the task-diverse case. 
  \textbf{Left:} Plot of the learning trajectories of $u_{\text{ind}}(t) = v(t)\lin^2(t)$ and $u_{\text{over}}(t) = v(t)\lov^2(t)$.    $u_{\text{ind}}(t)$ stays close to its initialization until $u_{\text{over}}(t)$ reaches $1$, at which point $u_{\text{ind}}(t)$ then begins to learn. This demonstrates the power-law relationship between $\lin$ and $\lov$, and hence two-phase learning dynamics.
 \textbf{Right:} Plot of the gradient norm of $v$ for simulated GF dynamics. $t_{\text{mix, low}}$ denotes the total time computed using the lower bound on $t_{\text{ind}}$. While the total times with both the upper and lower bounds yield values that are less than $t_{\text{single}}$, $t_{\text{mix, low}}$ serves as an accurate approximation.}
  \label{fig:power_law_pred_time}
\end{figure}

The full derivation is deferred to \Cref{sec:strict_acc_proof}, and a comparison of the leading-order terms without the approximation is provided in \Cref{sec:compare_convergence}.  
The approximation in \Cref{res:strict_acceleration} follows by taking the small-initialization limit $\alpha \to 0$: in this regime, $\lov^\star \approx \nov^{-1/6}$ by \Cref{eqn:main_lambda_star_eqn}, the lower-order terms in the growth times become negligible, and the leading growth-time ratio simplifies to $\rho(\kappa)$. Thus, \Cref{res:strict_acceleration} states that task diversity yields a strict reduction in the leading growth bottleneck whenever the two task subspaces have both shared and non-overlapping components, thereby leading to overall convergence acceleration.

Intuitively, this speedup arises from the two-stage learning phenomenon in the task-diverse setting. As we show in the following sections, the overlapping subspace components carry amplified signal strengths and therefore grow faster than the non-overlapping components, quickly driving $v$ toward its stationary value. Once $v$ is large, the remaining $\nin$ non-overlapping subspaces inherit a warm start from the shared structure rather than growing from near-zero. In contrast, in the single-task case, all $q$ subspace components must grow from the same small initialization simultaneously. This verifies the hypothesis of \cite{kim2025taskdiversityshortensicl}: the ICL plateau arises from learning the common task structures, and once those are learned, the individual task structures follow quickly. 

\subsection{Convergence Time with a Single Task}
\label{sec:single_subspace}

Here, we aim to derive $t_{\text{single}}$. Under normalization, the signal strengths in the single-task case are $\sigma_i = 1/q$ for $i \in [q]$ and $\sigma_i = 0$ for the remaining $d-q$ singular values. Plugging this into the dynamics in \Cref{eqn:final_gf} gives
\begin{align}
\label{eqn:no_task_gf}
    \tau \underline{\dot{v}} =  \left(1 - v\lambda^2 \right)\cdot \lambda^2 \quad \text{and} \quad 
    \tau \underline{\dot{\lambda}} = \frac{1}{q} \left(1- v \lambda^2 \right) \cdot v  \lambda,
\end{align}
where we dropped the subscript since the top-$q$ elements follow the same learning dynamics.
Note that the condition $v\lambda^2 = 1$ is sufficient for the entire system to be stationary, so we can focus on the number of iterations required for $\lambda$ to reach the stationary point. This involves computing the following integral:
\begin{align}
\label{eqn:t_sing_decomp}
    t_{\text{single}} = \tau q \int_{\alpha}^{\lambda^\star - \delta} \frac{d\lambda}{(1-v\lambda^2) \cdot v\lambda}  = \underbrace{\tau q \int_{\alpha}^{\lambda^\star - \delta} \frac{d\lambda}{v\lambda}}_{\text{growth time}} + \underbrace{\tau q \int_{\alpha}^{\lambda^\star - \delta} \frac{\lambda}{1 - v\lambda^2} d\lambda}_{\text{saturation time}},
\end{align}
where $\lambda^\star$ denotes the optimal solution. \Cref{eqn:t_sing_decomp} shows that the convergence time can be split into growth and saturation times: the growth time captures the rapid learning phase where the weights grow exponentially from near-zero initialization, and the saturation time captures the time for the system to slow down to a stationary point without overshooting. Using the conservation law in  \Cref{lemma:conservation},
both of these times admit closed-form integrals, which yields the following overall convergence time (see \Cref{sec:sing_subspace_time} for the full derivation):
\begin{align}
\label{eqn:t_sing_total}
t_{\text{single}} = \underbrace{\tau \sqrt{q} \left(\frac{1}{\alpha} - \frac{1}{q^{-1/6} - \delta}\right)}_{= t_{\text{single, growth}}} + \underbrace{\tau q^{2/3} \left(h(1-q^{1/6}\delta) - h(q^{1/6}\alpha) \right)}_{= t_{\text{single, sat}}},
\end{align}
where we plugged in the optimal solution $\lambda^\star = q^{-1/6}$ and $h(x)$ is the function 
\begin{align}
\label{eqn:h_x}
    h(x) \coloneqq -\frac{1}{3} \log |x - 1| + \frac{1}{6} \log(x^2 + x + 1) - \frac{\sqrt{3}}{3} \arctan\left(\frac{2x + 1}{\sqrt{3}}\right).
\end{align}
This provides the convergence time as shown in \Cref{eqn:main_sing_decomp}. 
Again, the primary bottleneck in \Cref{eqn:t_sing_total} is the growth-time term, since $\alpha$ is small and $q \in \mbb{N}$ can be arbitrarily large (assuming $d$ is also large). This term governs the ``ICL plateau'' often observed during transformer training~\cite{zhangtraining, kim2025taskdiversityshortensicl}, and scales with $q$ since all $q$ components of the training subspace $\mbf{U}_q \in \mbb{R}^{d\times q}$ must be learned from scratch.

\subsection{Convergence Time with Task Diversity: A Two-Stage Learning Paradigm}
\label{sec:two_stage_paradigm}

Following the same strategy as in \Cref{sec:single_subspace}, we determine the signal strengths $\sigma_i$ for the task-diverse case to derive $t_{\text{mix}}$. The mixture of two low-rank Gaussians in \Cref{eq:two_setup} yields the covariance matrix $\mbf{\Sigma} = \frac{1}{2}\cdot \mbf{\Sigma}_{s, 1} + \frac{1}{2} \cdot \mbf{\Sigma}_{s, 2}$. With $\nov$ overlapping components and $\nin \coloneqq q - \nov$ non-overlapping components, the spectrum and induced dynamics are
\[
\vcenter{\hbox{$
\mbf{\Lambda}_s =
  \begin{bmatrix}
    a^{-1}\cdot\mbf{I}_{n_{\text{over}}} & & \\
    & (2a)^{-1}\cdot \mbf{I}_{n_{\text{ind}}} & \\
    & & \mbf{0}_{d-q}
  \end{bmatrix}
$}}
\implies
\vcenter{\hbox{$
\begin{aligned}
&\tau \dot{\underline{\lambda}}_i = \frac{1}{a}(1 - v\lambda_i^2)\,v\lambda_i, \quad i \in [n_{\text{over}}],\\
&\tau\dot{\underline{\lambda}}_j = \frac{1}{2a}(1 - v\lambda_j^2)\,v\lambda_j, \quad j \in [n_{\text{ind}}],\\
&\tau\dot{\underline{v}} = \frac{1}{a}\sum_{i=1}^{n_{\text{over}}}(1 - v\lambda_i^2)\lambda_i^2
  + \frac{1}{2a}\sum_{j=1}^{n_{\text{ind}}}(1 - v\lambda_j^2)\lambda_j^2,
\end{aligned}
$}}
\]
where $a = \frac{2n_{\text{over}} + n_{\text{ind}}}{2}$ is the normalization constant. Since all the $\lambda_i$ start from the same initial condition, we can drop subscripts and track the following dynamics:
\begin{align*}
\tau \underline{\dot{\lambda}}_{\text{over}} &= \frac{1}{a} \left( 1 - v\lov^2 \right) v\lov \\
\tau \underline{\dot{\lambda}}_{\text{ind}} &= \frac{1}{2a} \left( 1 - v\lin^2 \right) v\lin \\
\tau \underline{\dot{v}} &= \frac{\nov}{a} \left( 1-v\lov^2  \right) v\lov + \frac{\nin}{2a} \left( 1-v\lin^2  \right) v\lin.
\end{align*}
These differential equations reveal interesting insights into the learning dynamics: the growth rate of $\lov$ is twice that of $\lin$, while $v$ couples the dynamics together. The faster growth rate of $\lov$ arises from the fact that the shared directions have their signal strengths amplified.
The coupled nature makes the difference in growth rates pivotal: the dynamics of $v$ are dominated by $\lov$ early on, with $\lin$ only beginning to contribute once $\lov$ saturates. To see this more clearly, consider the following ratio in the early stages of training:
\begin{align*}
    \frac{\underline{\dot{\lambda}}_{\text{over}}}{\dot{\underline{\lambda}}_{\text{ind}}} = \frac{\frac{1}{a} v\lambda_{\text{over}} + \mc{O}(\alpha^5)}{\frac{1}{2a} v\lambda_{\text{ind}} + \mc{O}(\alpha^5)} \implies \frac{\underline{\dot{\lambda}}_{\text{over}}}{\dot{\underline{\lambda}}_{\text{ind}}} \approx \frac{2\lambda_{\text{over}}}{\lambda_{\text{ind}}} \implies \frac{\underline{\dot{\lambda}}_{\text{over}}}{2\lambda_{\text{over}}} = \frac{\underline{\dot{\lambda}}_{\text{ind}}}{\lambda_{\text{ind}}}.
\end{align*}
Then, integrating both sides with some algebraic manipulation gives the following expression:
\begin{align*}
    \int\frac{\underline{\dot{\lambda}}_{\text{over}}}{2\lambda_{\text{over}}} = \int\frac{\underline{\dot{\lambda}}_{\text{ind}}}{\lambda_{\text{ind}}} \implies \ln\left( \lambda_{\text{ind}} \right) = \frac{1}{2} \ln\left( \lambda_{\text{over}} \right) + C \implies \lambda_{\text{ind}} = e^C \cdot \sqrt{\lambda_{\text{over}}}
\end{align*}
for some constant $C$. To determine the constant, we can plug in the initial conditions to obtain $\lambda_{\text{ind}} = \sqrt{\alpha \lambda_{\text{over}}}$.
This reveals an approximate power-law relationship between $\lin$ and $\lov$. For the system to be stationary, we require both $u_{\text{over}} \coloneqq v\lambda_{\text{over}}^2 = 1$ and $u_{\text{ind}} \coloneqq v\lambda_{\text{ind}}^2 = 1$. Since $\lambda_{\text{over}}$ has a faster growth rate, consider the moment at which $u_{\text{over}} = 1$. Then, the power-law relation gives $u_{\text{ind}} = v\lambda_{\text{ind}}^2 \approx v\alpha \lambda_{\text{over}} = \alpha / \lambda_{\text{over}}$.
This implies that $u_{\text{ind}}$ (and hence $\lambda_{\text{ind}}$) remains close to its initial condition until $\lambda_{\text{over}}$ saturates, yielding a two-stage learning phenomenon: (i) an initial phase in which the shared (or common) subspace components are learned, followed by (ii) a second phase in which the individual components are learned (see \Cref{fig:power_law_pred_time}). 

Consequently, this gives the decomposition of $t_{\text{mix}}$ into the two parts: $t_{\text{mix}} = t_{\text{over}} + t_{\text{ind}}$. These times are computed in closed-form using the power-law relation together with an ansatz on the dynamics:
\begin{align}
\label{eqn:common_time}
    t_{\text{over}}  
    &= \underbrace{\frac{2\tau a}{\alpha \nin} \left(\sqrt{q} - \sqrt{\nov + \frac{\nin \alpha}{\lov^\star-\delta}}\right)}_{= t_{\textup{over, growth}}} + \underbrace{a \tau \left( \lov^{\star} \right)^2\cdot \left(h\left( \frac{\lov^\star - \delta}{\lov^\star}\right) - h\left(\frac{\alpha}{\lov^\star} \right)\right)}_{= t_{\textup{over, sat}}},
\end{align}
\begin{align}
\label{eqn:t_mix_ind}
    \tau a \cdot \log\left( \frac{(1-\delta)(1 - \alpha/\lov^\star)}{\delta \alpha/\lov^\star}\right)g(1-\delta)  \;\leq\; t_{\text{ind}} \;\leq\; \tau a \cdot \log\left( \frac{(1-\delta)(1 - \alpha/\lov^\star)}{\delta \alpha/\lov^\star}\right)g(\alpha/\lov^\star),
\end{align}
where $h(\cdot)$ is defined in \Cref{eqn:h_x}, $\lov^\star$ is the positive root of 
\begin{align}
\label{eqn:lambda_star_eqn}
     \nov \left(\lov^\star\right)^6 + \alpha \nin \left(\lov^\star\right)^5 - 1 = 0 \quad \text{and} \quad     g(u) = \frac{\left( n_{\text{over}} + n_{\text{ind}} u\right)^{2/3}}{n_{\text{over}} + 2n_{\text{ind}} u} \quad \text{for} \quad u\in(0, 1),
\end{align}
which gives us the times stated in \Cref{res:time_decomp}. For brevity, we defer the full derivation to \Cref{sec:mix_subspace_time}.

Overall, \Cref{res:time_decomp} demonstrates the benefit of task diversity, as the two-stage learning dynamics accelerate convergence. However, notice that when there is no shared structure (i.e., $\kappa = 0$), there is no signal amplification, and the training time matches the single-task case. This raises the question: what conclusions, if any, can we draw about the converged solution? In the following section, we show that for any $\nov \geq 0$, the converged solution enables OOD capabilities for ICL.

\section{Task Diversity Enables Out-of-Distribution Generalization}
\label{sec:ood}

Following our previous discussion, in this section we study how training-task diversity affects the converged transformer weights. In particular, we demonstrate how training with diverse task vectors can enable transformers to exhibit OOD ICL capabilities. For ICL, there are several types of distribution shifts: covariate shifts (i.e., shifts in the inputs), task-function shifts (i.e., shifts in the function that generates outputs), and query shifts (i.e., the query input differs from the inputs in the test prompt). Since we focus on the effects of training with diverse task vectors, we restrict our attention to distribution shifts in the task vector while keeping other components fixed. To show this, we construct a task vector with zero density under the training distribution and prove that the test risk on this constructed task vector is negligible, indicating that ICL can generalize OOD.

\subsection{Training and Testing Data Distributions}
\label{sec:ood_train_test_dist}

For the training task vectors, we again consider the task-diverse
setup from \Cref{def:training_tasks} with $K=2$ to facilitate
discussion, and then generalize our results to the $K>2$ case. Let $\mbf{U}_{s, 1}, \mbf{U}_{s, 2} \in \mbb{R}^{d\times r}$ be the two (training) orthonormal bases such that $\mc{U}_s = \mathrm{span}(\mbf{U}_{s, 1}) +  \mathrm{span}(\mbf{U}_{s, 2})$,  where $\mathrm{dim}(\mc{U}_s) = q$ and $r = q/2$.\footnote{Note that this assumes that $\nov = 0$. For the case in which $\nov > 0$ (i.e., $\mathrm{rank}(\mbf{U}_{s, 1}) = \mathrm{rank}(\mbf{U}_{s, 2}) = w$ for some $w > q/2$), we can WLOG always construct orthonormal bases $\widehat{\mbf{U}}_{s, 1}, \widehat{\mbf{U}}_{s, 2} \in \mbb{R}^{d\times r}$ such that  $\widehat{\mbf{U}}_{s, 1} \perp \widehat{\mbf{U}}_{s, 2}$ and $\widehat{\mc{U}_s} = \mathrm{span}(\widehat{\mbf{U}}_{s, 1}) +  \mathrm{span}(\widehat{\mbf{U}}_{s, 2})$ with $\mc{U}_s=\widehat{\mc{U}_s}$.} Our goal is to define a testing subspace $\mbf{U}_t \in \mbb{R}^{d \times r}$ such that the train task vector $\mbf{w} \in \mbb{R}^d$ is sampled from $\mbf{U}_t$ with probability zero. Then, if we sample a testing task vector $\widetilde{\mathbf{w}}$ from $\mathbf{U}_t$, and the trained transformer achieves negligible test risk with respect to $\widetilde{\mathbf{w}}$, we can conclude that transformers are capable of OOD generalization via ICL.
To this end, we parameterize $\mbf{U}_t$ as such \cite[Section~3.8]{absil2004riemannian}: 
\begin{align}
\label{eq:task_test_subspace}
\mbf{U}_t \coloneqq \mbf{U}_t(\mbf{\Theta}) = \mbf{U}_{s, 1}\cdot \mathrm{cos}\big(\mbf{\Theta}\big) + \mbf{U}_{s,2} \cdot \mathrm{sin}\big( \mbf{\Theta} \big) \in \mbb{R}^{d\times r}.
\end{align}
Here, $\mbf{\Theta} \in \mbb{R}^{r \times r}$ is a diagonal matrix
with entries $\theta_i \in [0, \pi/2]$, and $\cos(\cdot)$ and $\sin(\cdot)$
are applied entrywise to the diagonal of $\mbf{\Theta}$. For simplicity, we will assume all principal angles are equal, i.e., for all $i \in [r]$, $\theta_i = \theta$ for some $\theta \in \left[ 0, \frac{\pi}{2} \right]$ so that
$\mbf{\Theta} = \theta \cdot \mbf{I}_r$. 
Notice when $\theta = 0$, $\mbf{U}_t = \mbf{U}_{s, 1}$, and when $\theta = \frac{\pi}{2}$, $\mbf{U}_t = \mbf{U}_{s,2}$. Intuitively, $\mathbf{U}_t$ defines an ``interpolating path'' between the training subspaces through $\theta$. When $\theta \in (0, \pi/2)$, $\mathrm{span}(\mathbf{U}_t)$ is distinct from both $\mathrm{span}(\mathbf{U}_{s, 1})$ and $\mathrm{span}(\mathbf{U}_{s, 2})$, ensuring that the test-task vector has zero probability of being sampled from the train subspace.
Now, we can set the testing covariance matrix as follows:
\begin{align}
\label{eqn:covariance_t}
    \mbf{\Sigma}_t &= \mbf{U}_t \mbf{U}_t^\top + \epsilon \cdot \mbf{I}_d, 
\end{align}
and each testing pair $(\mbf{x}_j, \widetilde{y}_j)$ is generated
independently of the training data: for all $j \in [m+1]$,
$\mbf{x}_j \sim \mc{N}(\mbf{0}, \mbf{I}_d)$ and
\begin{align}
\label{eq:vanilla_setup_test}
\widetilde{y}_j = \widetilde{\mbf{w}}^\top \mbf{x}_j + \xi_j,
\quad \text{where} \quad\widetilde{\mbf{w}} \sim \mc{N}(\mbf{0}, \mbf{\Sigma}_t),
\end{align}
with $\xi_j \sim \mc{N}(0, \sigma^2)$.

\subsection{ICL Can Generalize to the Span of Training Subspaces}

In this section, we consider the optimal single-layer linear attention model trained using the setup in \Cref{sec:ood_train_test_dist} and derive the test risk when the test prompts are drawn from \Cref{eq:vanilla_setup_test} for any angle $\theta \in [0, \tfrac{\pi}{2}]$. The following result shows that for sufficiently large prompt lengths, the test risk is independent of $\theta$ and achieves a value that is only a function of the irreducible noise variance.

\begin{tcolorbox}
    \begin{theorem}
\label{thm:pos_result}
    Let $g_{\mc{W}}^\star$ denote the optimal linear attention model corresponding to the independent training task diverse data setting in \Cref{def:training_tasks} with $K=2$. For all $j \in [m+1]$, suppose that the test prompts are constructed as in \Cref{eq:vanilla_setup_test}. Then, for any $\theta \in [0, \frac{\pi}{2}]$, we have
    \begin{equation*}
        \lim\limits_{m \to \infty}
        \lim\limits_{n \to \infty}
        \lim\limits_{\epsilon \to 0} \mathbb{E}\left[ \left( \widetilde{y}_{m+1} - g^\star_{\mc{W}}\left(\widetilde{\mathbf{Z}}  \right) \right)^2 \right] = \sigma^2.
    \end{equation*}
\end{theorem}
\end{tcolorbox}
Remarkably, \Cref{thm:pos_result} states that task diversity enables OOD generalization, as the optimal linear attention model trained using task vectors drawn from a mixture of Gaussians can generalize to all angles $\theta \in [0, \frac{\pi}{2}]$, or in other words, generalize to any test task vector drawn from the span of the training subspaces. In \Cref{fig:pos_result}, we corroborate \Cref{thm:pos_result} on both linear attention and GPT-2, demonstrating that test loss approaches zero for larger prompt lengths and that the result generalizes beyond linear attention. Hence, we hypothesize that this explains why ICL achieves OOD generalization: the test data actually lies within the span of the training data. 
We remark that in \Cref{thm:pos_result}, we take $\epsilon \to 0$ for two reasons: (i) to eliminate any dependence on $\epsilon$ and isolate its effect on test risk as it is assumed to be a small constant, and (ii) to analyze the test risk when the covariance matrices are exactly low-rank.

\begin{figure*}[t!]
    \centering
    \includegraphics[width=0.85\linewidth]{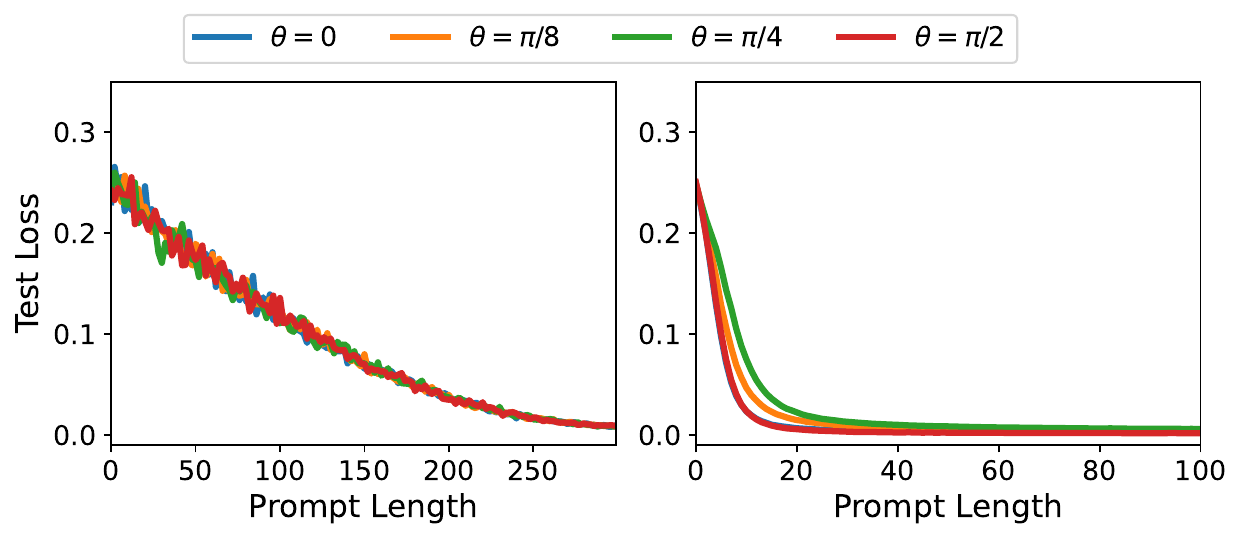}
    \caption{\textbf{Plot of the test risk as a function of the prompt length when trained using diverse task vectors with test subspace drawn from \Cref{eq:task_test_subspace}.} \textbf{Left:} Plot of the risk for linear attention. \textbf{Right:} Plot of the risk for GPT-2.  For both plots, when the prompt length at test time is large enough, the test risk goes nearly to zero for all $\theta \in \left[0, \frac{\pi}{2} \right]$, corroborating Theorem~\ref{thm:pos_result}. This shows that both  transformers can generalize to the span of the training task vectors at test-time. }
    \label{fig:pos_result}
\end{figure*}

Next, we generalize the above result to a mixture of $K > 2$
Gaussians. To this end, let us assume $d \geq Kr \eqqcolon q$. Then, for all $k \in [K]$, we define $\mbf{U}_{s, k} \in \mbb{R}^{d \times r}$ such that  $\mbf{U}_{s, k}^\top \mbf{U}_{s, \ell} = \mbf{0}$ for all $k \neq \ell$. Then, we assume the training task $\mbf{w} \in \mbb{R}^d$ is sampled as such:
\begin{align}
    \label{eqn:task_multiple_mog}
    &\mbf{w} \sim \sum\limits_{k=1}^K \gamma_k \cdot \mathcal{N}(\mbf{0}, \mbf{\Sigma}_{s, k}), \quad \text{where} \quad \mbf{\Sigma}_{s, k} = \mbf{U}_{s, k}\mbf{U}_{s, k}^\top + \epsilon \cdot \mbf{I}_d \quad \text{and} \quad \sum\limits_{k=1}^K \gamma_k = 1.
\end{align}
Similarly, we define an orthonormal testing basis $\overline{\mbf{U}}_t \in \mbb{R}^{d \times r}$ that lies within the span of $\{\mbf{U}_{s, k}\}_{k=1}^K$:
\begin{equation} \label{eq:U_t_bar}
    \overline{\mbf{U}}_t = \sum\limits_{k=1}^K \alpha_k \mbf{U}_{s, k}, \quad \text{for} \quad  \{\alpha_k\}_{k=1}^K \quad \text{such that} \quad  \sum\limits_{k=1}^K \alpha_k^2 = 1.
\end{equation}
The constraint on $\{\alpha_k\}_{k=1}^K$ ensures $\overline{\mbf{U}}_t \in \mbb{R}^{d\times r}$ is an orthonormal basis. Then,  we consider testing on task vectors 
\begin{align*}
    \widetilde{\mbf{w}} \sim \mathcal{N}\left(\mbf{0}, \overline{\mbf{\Sigma}}_t \right), \quad \text{where} \quad \overline{\mbf{\Sigma}}_t = \overline{\mbf{U}}_t\overline{\mbf{U}}_t^\top + \epsilon \cdot \mbf{I}_d.
\end{align*}
Again, we emphasize $\overline{\mbf{U}}_t$ is unseen during training, but lies within the span of the training subspaces.

\begin{tcolorbox}
    \begin{theorem}
\label{thm:mixture_k_subspaces}
Let $g^\star_{\mc{W}}$ denote the optimal linear attention model
corresponding to the independent training task diverse data setting in \Cref{def:training_tasks}, where the task vector is drawn from
\Cref{eqn:task_multiple_mog} with $\gamma_k = \tfrac{1}{K}$ for all
$k \in [K]$. For all $j \in [m+1]$, suppose the test prompts are
constructed with features $\mbf{x}_j \sim \mc{N}(\mbf{0}, \mbf{I}_d)$
and labels whose task vectors are parameterized using the basis $\overline{\mbf{U}}_t$ defined in  \Cref{eq:U_t_bar}. For any
$\{\alpha_k\}_{k=1}^{K}$ with $\sum_{k=1}^{K} \alpha_k^2 = 1$, we have
\begin{equation*}
    \lim_{m \to \infty} \lim_{n \to \infty} \lim_{\epsilon \to 0}
    \mbb{E}\!\left[ \left( \widetilde{y}_{m+1}
    - g^\star_{\mc{W}}\left(\widetilde{\mbf{Z}} \right) \right)^2 \right]
    = \sigma^2.
\end{equation*}
\end{theorem}
\end{tcolorbox}
Similar to \Cref{thm:pos_result}, if the linear attention model is trained on task vectors that lie in a union of $K$ subspaces, it can generalize well to any region within the span of the $K$ subspaces, even if those regions have zero probability density during training. 
Lastly, note that by setting $K = 2$, $\alpha_1 = \cos(\theta)$, and $\alpha_2 = \sin(\theta)$, we exactly recover \Cref{thm:pos_result}.

\begin{figure*}[t!]
    \centering
    \includegraphics[width=0.925\linewidth]{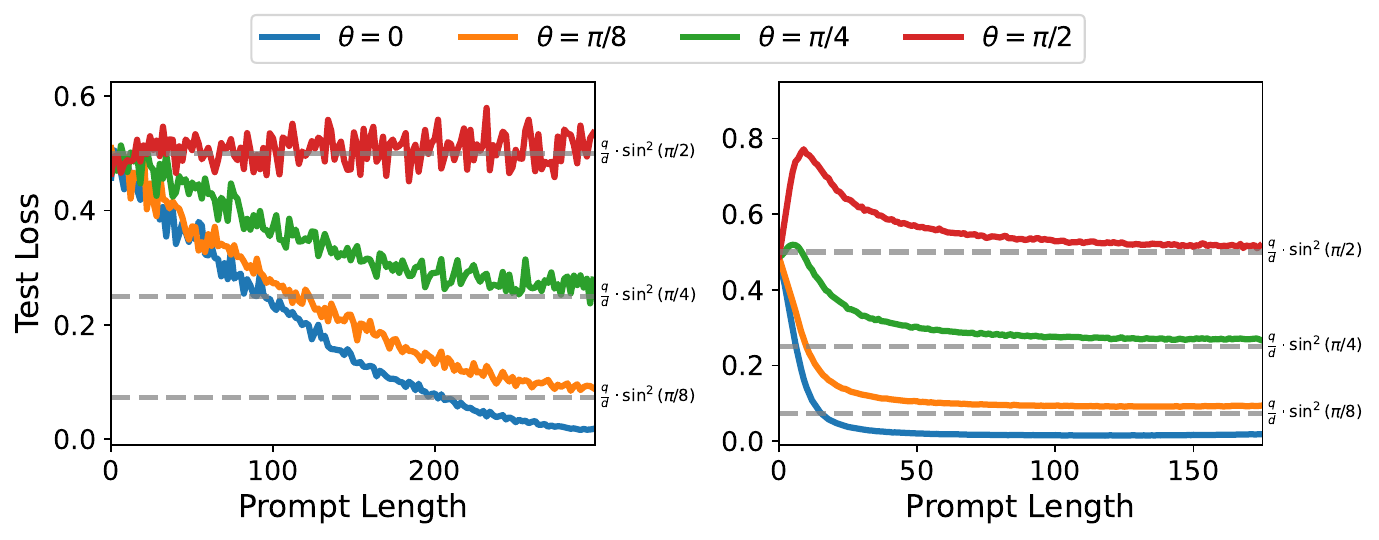}
    \caption{\textbf{Plot of the test risk as a function of prompt length when trained using diverse task vectors, but with the test subspace drawn from \Cref{eq:outside_task_subspace}.} \textbf{Left:} Plot of the risk for linear attention. \textbf{Right:} Plot of the risk for GPT-2. The test loss does not approach zero even with large prompt lengths, and  instead converges to the (normalized) test risk in \Cref{thm:neg_result} as the test subspace shifts away from the training subspaces at angle $\theta$ for both architectures.}
    \label{fig:neg_result}
\end{figure*}

\subsection{ICL Cannot Generalize Outside of the Training Subspaces}

Previously, we saw that when a transformer is trained with diverse task vectors, transformers can generalize to any test task vector drawn from the span of the training subspaces via ICL, despite not being present in the training data. This leaves us with the following question: what happens if we test the model with a subspace outside of the training subspaces, and what is its effect on the test risk? To investigate this, we can similarly define a test subspace at an angle outside of the training subspaces, and consequently compute the test risk with respect to this new test subspace.

To this end, with a slight abuse of notation, let us go back to the $K=2$ case and define a new testing subspace
\begin{align}
\label{eq:outside_task_subspace}
\mbf{U}_t = \mbf{U}_{s}\cdot \mathrm{cos}\big(\mbf{\Theta}\big) + \mbf{U}_{s,\perp} \cdot \mathrm{sin}\big( \mbf{\Theta} \big) \in \mbb{R}^{d\times q},
\end{align}
where $\mbf{U}_s \in \mbb{R}^{d\times q}$ is an orthonormal basis for the span of training subspaces $\mc{U}_s = \mathrm{span}(\mbf{U}_{s, 1}) +  \mathrm{span}(\mbf{U}_{s, 2})$ and $\mbf{U}_{s, \perp}\in \mbb{R}^{d\times q}$ is any arbitrary $q$-dimensional subspace orthogonal to $\mbf{U}_s$, i.e., $\mbf{U}_s^\top \mbf{U}_{s,\perp} = \mbf{0}$. The following result presents the test risk
with respect to $\mbf{U}_t$.

\begin{tcolorbox}
    \begin{theorem}
\label{thm:neg_result}
Let $g^\star_{\mc{W}}$ denote the optimal linear attention model
corresponding to the independent training task diverse data setting in \Cref{def:training_tasks} with $K=2$.
For all $j \in [m+1]$, suppose that the test prompts are constructed as in \Cref{eq:vanilla_setup_test}, where the testing subspace $\mbf{U}_t$ takes the form in \Cref{eq:outside_task_subspace}. Then, for some $\theta \in [0, \frac{\pi}{2}]$, we have
\begin{equation*}
    \lim_{m \to \infty} \lim_{n \to \infty} \lim_{\epsilon \to 0}
    \mbb{E}\!\left[ \left( \widetilde{y}_{m+1}
    - g^\star_{\mc{W}}(\widetilde{\mbf{Z}}) \right)^2 \right]
    = q \sin^2(\theta) + \sigma^2.
\end{equation*}
\end{theorem}
\end{tcolorbox}

In contrast to \Cref{thm:pos_result} and \Cref{thm:mixture_k_subspaces}, \Cref{thm:neg_result} states that when the testing subspace is at an angle $\theta$ away from the training subspace, the test risk is a function of $\theta$ scaled by the subspace rank $q$. Of course, when $\theta = 0$, we have $\mbf{U}_t = \mbf{U}_s$, and the test risk reduces to that of \Cref{thm:pos_result}, but for any $\theta \in (0, \frac{\pi}{2}]$, the test risk has a non-negligible dependence on $\theta$. This result further implies that OOD generalization in this setup can only occur when the test subspace lies within the span of the training subspaces (and hence when the train task vectors are diverse), and so the apparent OOD capabilities arise precisely from that condition. In \Cref{fig:neg_result}, we corroborate our result on both linear attention and GPT-2, showing that this limitation does not arise from analyzing linear attention but is a fundamental limitation of ICL.
\section{Experimental Results}
\label{sec:experiments}

This section is organized as follows: in \Cref{sec:exp_setup}, we describe the experimental setup for both linear attention and GPT-2; in \Cref{sec:exp_convergence}, we provide results related to the acceleration analysis in \Cref{sec:acceleration}; and finally, in \Cref{sec:exp_ood}, we present our findings regarding OOD generalization as in \Cref{sec:ood}.

\begin{figure}[t!]
    \centering
    \includegraphics[width=0.8\linewidth]{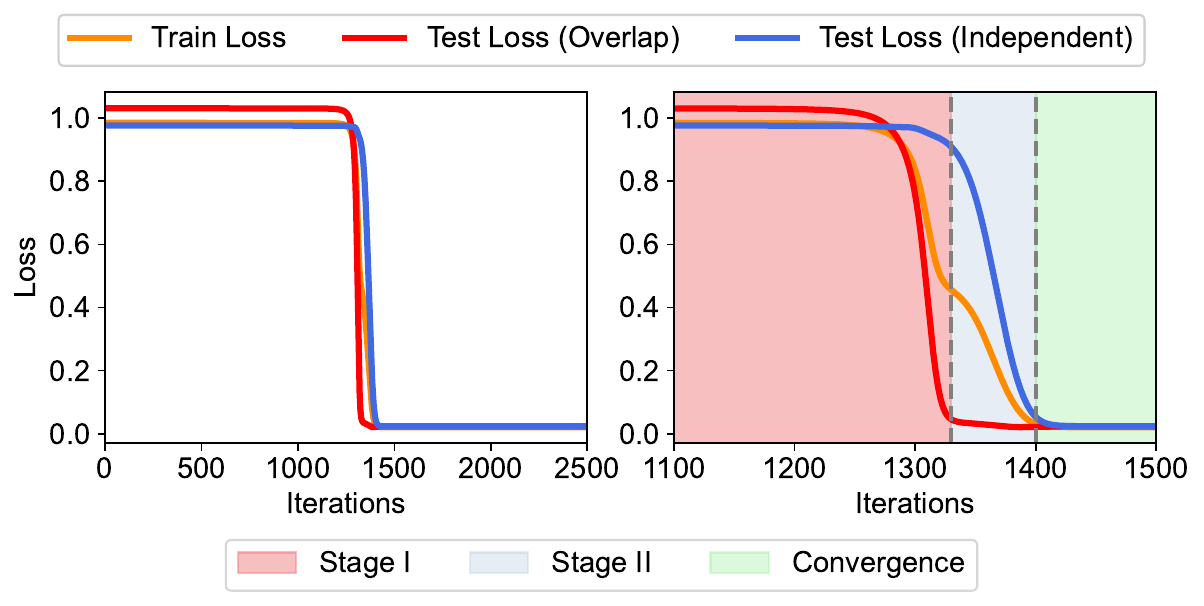}
    \caption{\textbf{Depiction of the two-stage learning phenomenon on a single-layer linear attention model.} We verify the two stages by testing on two different test-task vectors: one drawn from the overlapping subspace, and another drawn from the independent components (i.e., the remaining components orthogonal to the overlapping subspace). The right figure shows a zoomed-in version of the test losses, showing that the test loss corresponding to the overlapping subspace drops first, followed by a rapid drop of the test loss of the independent components, indicating the two-stage learning dynamics.}
    \label{fig:two_phase_learning}
\end{figure}

\subsection{Experimental Setup}
\label{sec:exp_setup}

Unless otherwise stated, the setup is as follows. For experiments with GPT-2, we follow Garg et al.~\cite{garg2022what} and use a model with 6 layers, 4 heads, and a 128-dimensional embedding space. We append a learnable linear transformation to map the vector predicted by the model to a scalar. We use a learning rate of $\eta = 10^{-4}$, a batch size of 128, and prompt lengths of $m = n = 150$, and we train for a total of 100K iterations. We train the model using a single A100 GPU.

For experiments with linear attention, the setup differs between \Cref{sec:exp_convergence} and \Cref{sec:exp_ood}. For \Cref{sec:exp_ood}, no training is performed, as we test the model using optimal linear attention weights that can be derived in closed form (see \cite[Appendix~B]{kwon2026outofdistribution}). For \Cref{sec:exp_convergence}, we train a single-layer transformer with a fixed training dataset size of 6000 and a test dataset size of 3000, where applicable. We use a learning rate of $\eta = 2.5 \times 10^{-3}$ and $\alpha = 0.001$, and we train for a total of 2500 iterations. We train the model on a MacBook Pro with an M2 chip.

\subsection{Experiments for Accelerating Convergence}
\label{sec:exp_convergence}

\paragraph{Linear Attention.} First, we provide experiments corroborating the two-stage learning dynamics discussed in \Cref{sec:acceleration}, where common subspace components are learned first, followed by the independent components. To this end, we train a single-layer linear attention model with $d=15$ and $q=12$. We construct training orthonormal bases $\mbf{U}_{s, 1}, \mbf{U}_{s, 2} \in \mathbb{R}^{15 \times 8}$ such that $\nov = 4$ and $\nin = 8$. Recall that for the single-task setting, we then use $\mbf{U}_s$, which is defined as an orthonormal basis for the subspace $\mathrm{span}(\mbf{U}_{s, 1}) + \mathrm{span}(\mbf{U}_{s, 2})$.
Let $\mbf{U}_{\text{over}} \in \mbb{R}^{d \times \nov}$ and $\mbf{U}_{\text{ind}} \in \mbb{R}^{d \times \nin}$ denote the subspaces spanned by the overlapping and independent components of  $\mbf{U}_{s, 1}$ and $\mbf{U}_{s, 2}$, respectively. To verify the two-stage phenomenon, we draw test task vectors from the subspaces $\mbf{U}_{\text{over}}$ and $\mbf{U}_{\text{ind}}$ and record the test loss during training. As shown in \Cref{fig:two_phase_learning}, the test loss corresponding to the overlapping subspace drops first, followed by the test loss corresponding to the independent subspace.

\begin{figure}[t!]
    \centering
     \begin{subfigure}[b]{0.49\textwidth}
         \centering
        \includegraphics[width=0.85\textwidth]{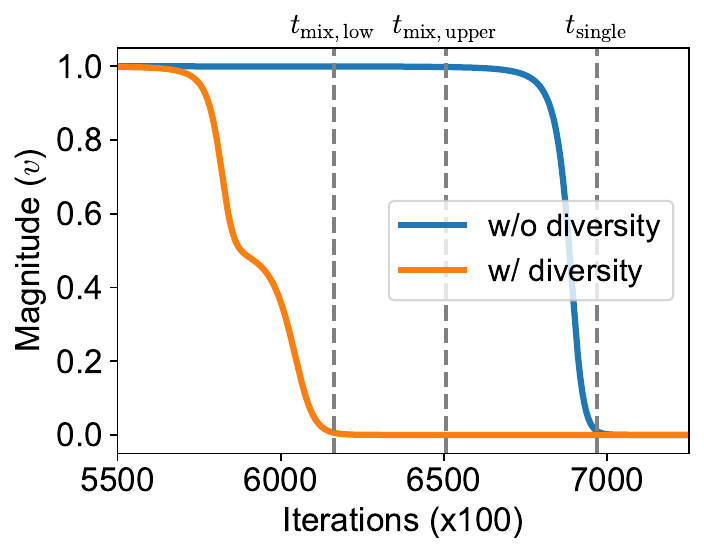}
     \end{subfigure}
     \hfill
     \begin{subfigure}[b]{0.49\textwidth}
         \centering
         \includegraphics[width=0.925\textwidth]{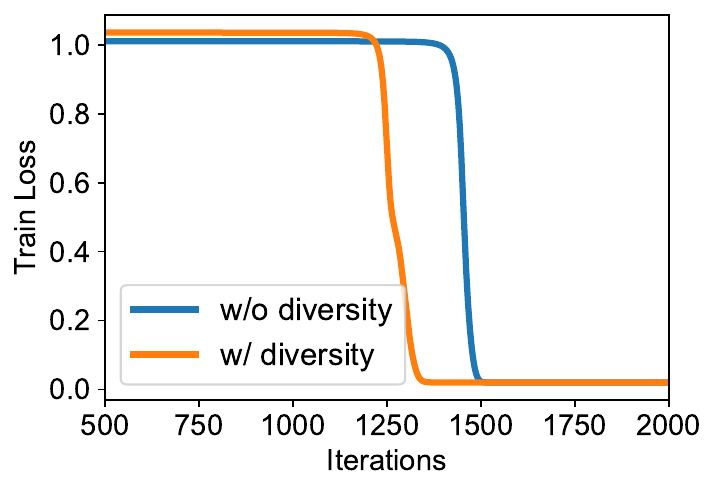}
     \end{subfigure}   
    \caption{\textbf{Demonstrating that task diversity shortens the ICL plateau and accelerates convergence.} We consider the case in which $\nov = 4$ and $\nin = 8$ (which sets $\nov=12$ without diversity). \textbf{Left:} Simulated GF dynamics demonstrating that $t_{\text{mix, low}}$, defined as $t_{\text{mix}}$ evaluated at the lower bound of $t_{\text{ind}}$, is strictly less than $t_{\text{single}}$ and serves as an accurate approximation of the convergence time. \textbf{Right:} Plot of the training loss for a single-layer linear attention model, showing that task diversity accelerates convergence.}
    \label{fig:gf_and_linear}
\end{figure}

\begin{figure}[t!]
    \centering
    \includegraphics[width=0.8\linewidth]{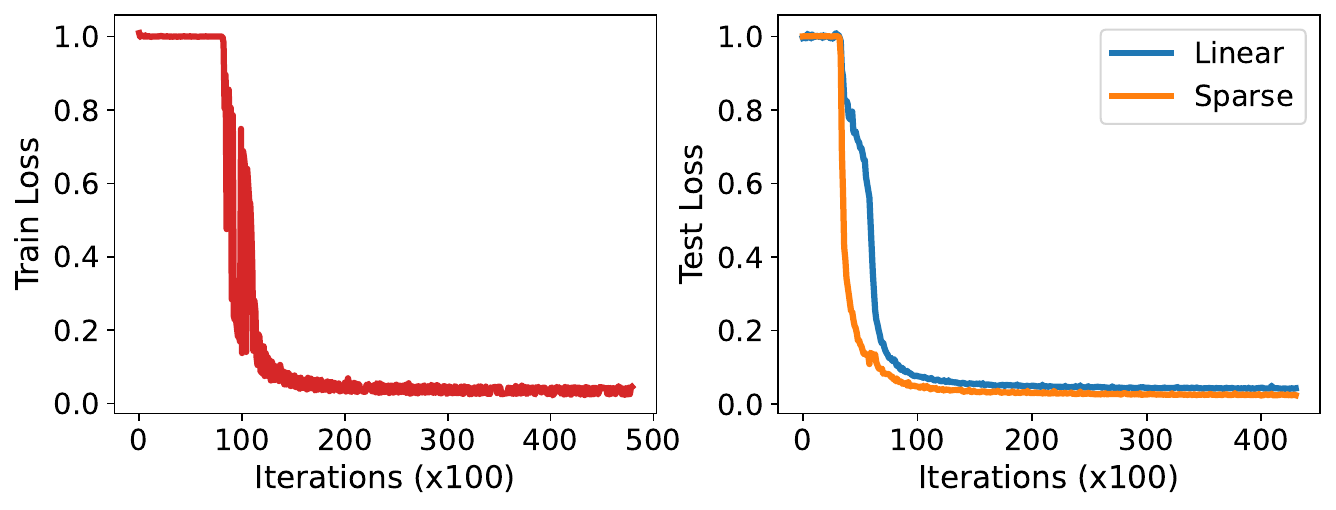}
    \caption{\textbf{Depiction of the two-stage learning phenomenon on GPT-2, where we train the model using two function classes, linear regression and sparse regression, sampled with equal probability.} For sparse regression, we mask out $5$ components such that the unmasked components can be viewed as the overlapping components. The ICL plateau drops once the sparse regression components (i.e., the common components) are learned, after which the test loss for linear regression drops shortly thereafter.}
    \label{fig:splinreg}
\end{figure}

Second, we consider the same setup as above to demonstrate that training with diverse task vectors speeds up convergence. In \Cref{fig:gf_and_linear} (left), we present results using simulated GF dynamics, similar to the approach in \Cref{fig:power_law_pred_time}. Here, $t_{\text{mix, low}}$ and $t_{\text{mix, upper}}$ denote the estimated convergence times for the task-diverse setting, evaluated using the lower and upper bounds of $t_{\text{ind}}$, respectively. We demonstrate that while both bounds are smaller than $t_{\text{single}}$, $t_{\text{mix, low}}$ serves as a tight approximation of the actual convergence time. In \Cref{fig:gf_and_linear} (right), we show the training loss of a linear attention model, where the loss decreases faster in the task-diverse case.

\paragraph{Softmax Attention.} Here, we aim to show that the two-stage learning phenomenon holds more generally, specifically for GPT-2. We consider a setting with $d=10$ and $n=150$, training the model on task vectors drawn with equal probability from two function classes: linear regression, where $\mbf{w}, \mbf{x}_i \sim \mc{N}(\mbf{0}, \mbf{I}_d)$, and sparse regression, where 5 components are masked. By viewing the sparse regression components as a subset of the full linear regression space (the overlapping components), one would expect the test loss for sparse regression to drop first, followed by the loss for linear regression. As shown in \Cref{fig:splinreg}, the test loss for linear regression exhibits a delayed drop, while the total ICL plateau begins to descend only after the sparse regression task is learned, i.e., the common task is learned.

\begin{figure}[t!]
    \centering
     \begin{subfigure}[t!]{0.49\textwidth}
         \centering
        \includegraphics[width=0.9\textwidth]{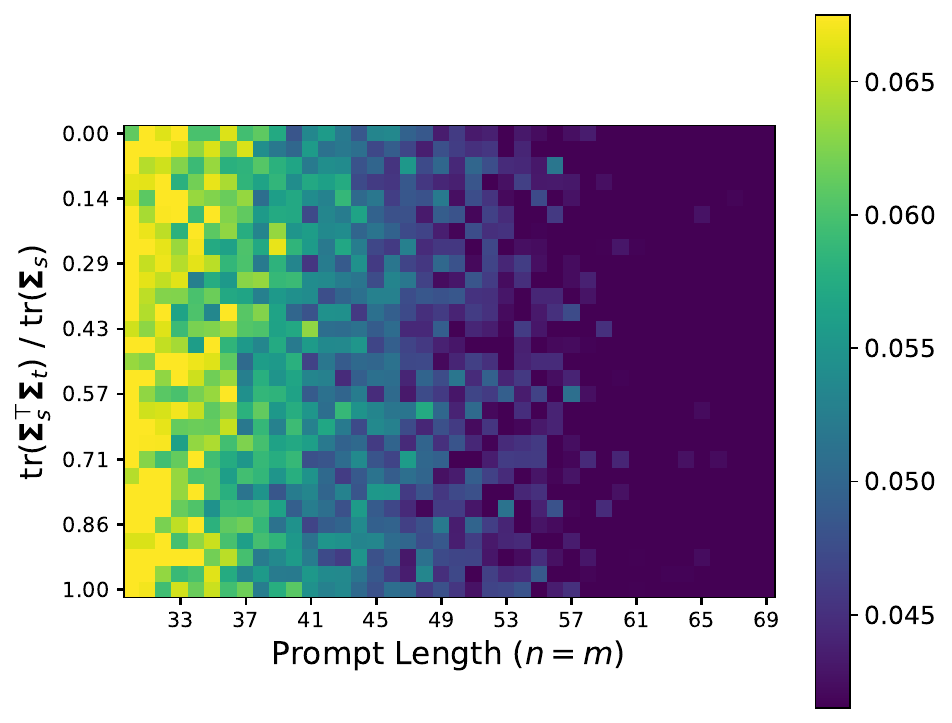}
     \end{subfigure}
     \hfill
     \begin{subfigure}[t!]{0.49\textwidth}
         \centering
         \includegraphics[width=0.8\textwidth]{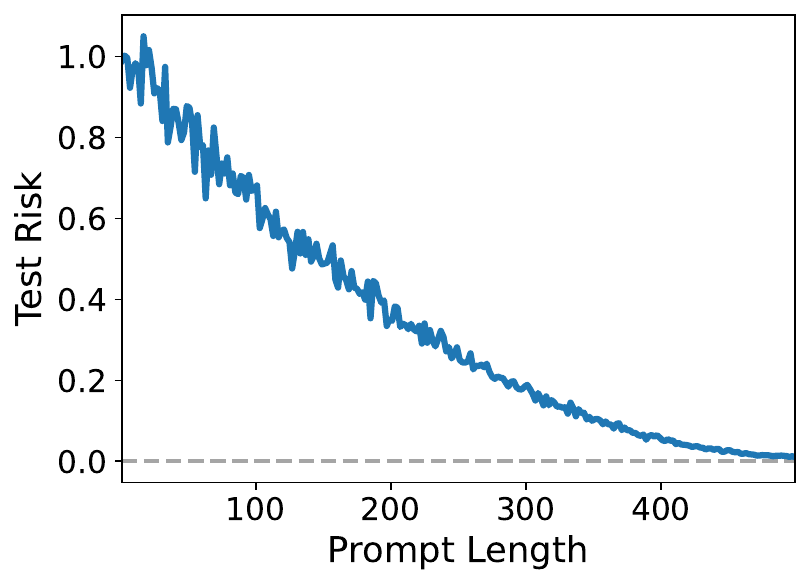}
     \end{subfigure}   
    \caption{\textbf{Left:} Phase plot of the test risk as we vary the angle between $\mbf{\Sigma}_s$ and $\mbf{\Sigma}_t$ and the prompt length with $m=n$ for a linear attention model trained with a mixture of Gaussians. The test risk is low across all angle shifts, and decreases further as the prompt length increases.  \textbf{Right:} Plot of the test risk as a function of the prompt length for a case in which $\mbf{\Sigma}_s \neq \mbf{\Sigma}_t$ but with $\theta = 0$, following the OOD example in~\cite{gatmiry2024looped}. This serves to explain why ICL can seemingly do OOD generalization as observed in the literature.}
    \label{fig:phase_and_unify}
\end{figure}

\subsection{Experiments for Out-of-Distribution Generalization}
\label{sec:exp_ood}

For experiments on both linear and softmax attention models (including \Cref{fig:pos_result} and \Cref{fig:neg_result}), we use $K=2$ with $d=20$, $q=10$, $r=5$, and $\sigma^2 = 0$. To construct the subspaces, we sample an orthogonal matrix $\mathbf{U} \in \mathbb{R}^{d \times d}$ uniformly at random, then set $\mathbf{U}_s$ as the first $q$ columns of $\mathbf{U}$ and $\mathbf{U}_{s, \perp}$ as the second $q$ columns. Consequently, $\mathbf{U}_{s, 1}$ and $\mathbf{U}_{s, 2}$ are formed by taking the first and last $r$ columns of $\mathbf{U}_s$, respectively.

\paragraph{Linear Attention.} 

To supplement \Cref{fig:pos_result}, in Figure~\ref{fig:phase_and_unify} (left), we present a phase plot of the test risk as a function of $ \tr(\mbf{\Sigma}_s^\top \mbf{\Sigma}_t) \, / \, \tr(\mbf{\Sigma}_s)$ (which measures the angle between two covariance matrices) and the prompt length on linear attention with task vectors drawn from a mixture of two Gaussians. Similar to Figure~\ref{fig:pos_result}, the test risk is low for all values of $m = n$, and it decreases further as the prompt length increases. Note that the largest possible normalized test risk in this setting is $r/d = 0.25$, so the test risk is still considered low even when the prompt length is small.

Next, we present an additional experiment supporting our primary message in \Cref{sec:ood} by adopting the setting from \cite{gatmiry2024looped}, with $\mathbf{\Sigma}_s = \mathbf{I}_5$ and $\mathbf{\Sigma}_t = \mathbf{V}\mathbf{\Lambda}_t \mathbf{V}^\top$, where $\mathbf{V} \in \mathbb{R}^{5\times 5}$ is a random orthogonal matrix and $\mathbf{\Lambda}_t = \mathrm{Diag}(1, 1, 1/2, 1/4, 1)$. As shown in Figure~\ref{fig:phase_and_unify} (right), the test risk approaches zero given a sufficient number of samples. This suggests that our results may help explain various observations of OOD generalization in ICL; specifically, because the testing covariance matrix is a subset of the training covariance matrix, our framework offers a unifying perspective on findings reported in the literature.

\begin{figure}[t!]
    \centering
     \begin{subfigure}[b]{0.49\textwidth}
         \centering
        \includegraphics[width=0.95\textwidth]{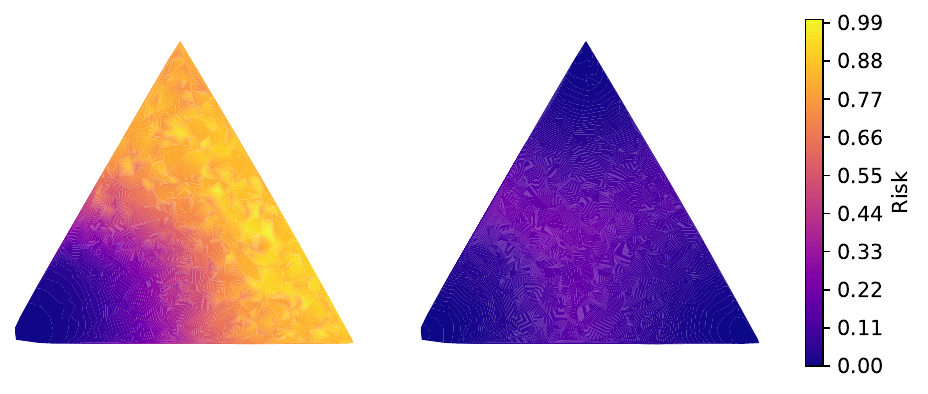}
     \end{subfigure}
     \hfill
     \begin{subfigure}[b]{0.49\textwidth}
         \centering
         \includegraphics[width=0.95\textwidth]{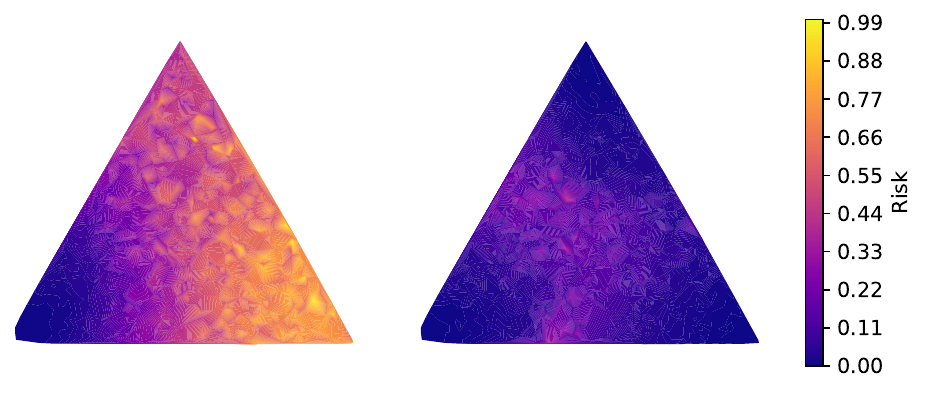}
     \end{subfigure}   
    \caption{\textbf{Visualization of the generalization behavior of transformers for learning nonlinear function classes in-context.} Each corner of a triangle represents a one-dimensional subspace spanned by $\psi_1$ (bottom left), $\psi_2$ (bottom right), or $\psi_3$ (top), with all possible convex combinations given by the interior. In all cases, we show the risk when evaluated at different points in $\mbox{span}(\{\psi_1, \psi_2, \psi_3\})$ for the appropriate function space. \textbf{Left:} Experiments using cosine bases. The first trains on prompts drawn from $\mbox{span}(\{\psi_1^C\})$, while the second trains on prompts drawn from $\mbox{span}(\{\psi_1^C\}) \cup \mbox{span}(\{\psi_2^C\}) \cup \mbox{span}(\{\psi_3^C\})$. \textbf{Right:} Experiments with Hermite polynomials. The third trains on prompts drawn from $\mbox{span}(\{\psi_1^H\})$. The fourth trains on prompts drawn from $\mbox{span}(\{\psi_1^H\}) \cup \mbox{span}(\{\psi_2^H\}) \cup \mbox{span}(\{\psi_3^H\})$. 
    }
    \label{fig:beyond-linear}
\end{figure}

\paragraph{Softmax Attention.} We use the GPT-2 model to extend our observations beyond linear function classes. Specifically, we look at two function spaces, namely $L^2([0, 1])$ and $L^2(\mathbb{R}, e^{-x^2/2} / \sqrt{2\pi} \,dx)$, i.e., square-integrable functions under the uniform and Gaussian measures respectively, which model rich sets of signals observed in real-world data.
For the former, we construct an orthonormal basis via cosines, i.e., $\psi_n^C(x) = (1/\sqrt{2})\cos(n\pi x)$ for $n \in \mathbb{N}$. For the latter, we construct an orthonormal basis via Hermite polynomials: 
\begin{equation*}
    \psi_n^H(x) = \frac{(-1)^n}{\sqrt{n!}} e^{x^2/2} \frac{d^n(e^{-x^2/2})}{dx^n} \quad \mbox{for} \; n \in \mathbb{N}.
\end{equation*}
As described in previous sections, we consider two settings: observing instances of a single (one-dimensional) subspace, as well as for a union of three (one-dimensional) subspaces. As before, we draw the function coefficients from standard multivariate Gaussian. We draw the inputs from the distribution appropriate to the function space measure, i.e., $x \sim \mathcal{U}([0, 1])$ for $L^2([0, 1])$ and $x \sim \mathcal{N}(0, 1)$ for $L^2(\mathbb{R}, e^{-x^2/2} / \sqrt{2\pi} \,dx)$. All other details are identical to previous (nonlinear) transformer experiments. The results are shown in \Cref{fig:beyond-linear}.

As shown in panels (a) and (c) of \Cref{fig:beyond-linear}, we see that transformers are not robust to subspace shifts for either function class, with increasing test risk with respect to the subspace angle from the train subspace, in accordance with \Cref{thm:neg_result}. On the other hand, as shown in panels (b) and (d) of \Cref{fig:beyond-linear}, we have the generalization behavior described by \Cref{thm:pos_result}, where training on the mixture of subspaces results in low risk in the space spanned by the basis vectors.

\section{Conclusion}
\label{sec:conclusion}

In this work, we proposed a new notion of task diversity, distinct from existing definitions in the literature, using low-dimensional subspaces. We showed that this provides a tractable framework for analysis and demonstrated how task diversity can both accelerate convergence to a stationary point and enable OOD generalization in ICL. For convergence, we showed that task diversity induces a two-stage learning paradigm: common structures are learned first, and once learned, the distinct components are learned much more quickly, thereby accelerating convergence. For OOD generalization, we showed that task diversity enables generalization to regions with zero probability density under the training distribution. Furthermore, we demonstrated empirically that our theoretical results from both parts extend to transformer models such as GPT-2.


\section*{Acknowledgement}
 QQ, SK, AX, and CY acknowledge NSF CAREER CCF-2143904, NSF IIS 2312842, NSF IIS 2402950, and DARPA HR00112520042. QQ also acknowledges the Google Research Scholar and Google TPU Award.  LB, CY, and SK acknowledge NSF CAREER CCF-1845076 and NSF CCF-2331590. We would like to thank Emrullah Ildiz (University of Michigan), Samet Oymak (University of Michigan), and Daniel Hsu (Columbia University) for fruitful discussions.

\newpage

\printbibliography

\newpage

\appendix

\onecolumn
\par\noindent\rule{\textwidth}{1pt}
\begin{center}
{\Large \bf Appendix}
\end{center}
\vspace{-0.1in}
\par\noindent\rule{\textwidth}{1pt}
\appendix


\section{Deferred Proofs from \Cref{sec:acceleration}}

\subsection{Proof of \Cref{prop:gf_dynamics}}

\begin{proof}
We begin by simplifying the linear attention model. Let us define the following:
\begin{align*}
    \mbf{W}_Q = \begin{bmatrix}
        \mbf{Q} & \mbf{q}_1 \\
        \mbf{q}_2^\top & q
    \end{bmatrix},
    \quad
    \mbf{W}_K = \begin{bmatrix}
        \mbf{K} & \mbf{k}_1 \\
        \mbf{k}_2^\top & k
    \end{bmatrix},
    \quad
    \mbf{W}_V = \begin{bmatrix}
        \mbf{V} & \mbf{v}_1 \\
        \mbf{v}_2^\top & v
    \end{bmatrix},
    \quad\text{and} \quad
    \mbf{p} = \begin{bmatrix}
    \mbf{0}_d \\ 1
\end{bmatrix},
\end{align*}
where $\mbf{Q}, \mbf{K}, \mbf{V} \in \mbb{R}^{d\times d}$. Then, note that the initialization conditions in \Cref{eqn:gf_initial} sets $\mbf{q}_i, \mbf{k}_i, \mbf{v}_i = \mbf{0}_d$ for $i \in \{1, 2\}$. Some algebraic manipulations allow us to simplify the model as such:
\begin{align*}
g_\mc{W}(\mbf{Z}) &= \frac{1}{n}\left(\mbf{z}_{q}^\top\mbf{W}_Q \mbf{W}_K^\top \mbf{Z}_\mc{M}^\top \right) \mbf{Z}_\mc{M} \mbf{W}_V \mbf{p} \tag{$\mbf{z}_q$ denotes query}\\
&=\frac{1}{n} \underbrace{\begin{bmatrix}
    \mbf{x}_q^\top & 0
\end{bmatrix}}_{\eqqcolon \mbf{z}_q^\top}
\mbf{W}_Q\mbf{W}_K^\top \mbf{Z}_\mc{M}^\top \mbf{Z}_\mc{M} \underbrace{\begin{bmatrix}
    \mbf{v}_1 \\ v
\end{bmatrix}}_{\coloneqq \mbf{v}} \tag{Due to $\mbf{0}_d$ in $\mbf{p}$} \\
&= \begin{bmatrix}
    \mbf{x}_q^\top & 0
\end{bmatrix}
\begin{bmatrix}
    \mbf{Q} & \mbf{q}_1 \\ \mbf{0}_d^\top & 0
\end{bmatrix}
\begin{bmatrix}
        \mbf{K}^\top & \mbf{k}_2 \\
        \mbf{k}_1^\top & k
    \end{bmatrix}
\underbrace{\begin{bmatrix}
\frac{1}{n}\sum_{i=1}^{n} \mathbf{x}_{i}\,\mathbf{x}_{i}^{\top} &
 \frac{1}{n}\sum_{i=1}^{n} y_{i}\,\mathbf{x}_{i} \\
\frac{1}{n}\sum_{i=1}^{n} y_{i}\,\mathbf{x}_{i}^{\top} &
\frac{1}{n}\sum_{i=1}^{n} y_{i}^{2}
\end{bmatrix}}_{= \mbf{Z}_{\mc{M}}^\top\mbf{Z}_{\mc{M}}} \mbf{v} \tag{Due to $0$ in $\mbf{z}_q$}\\
&= \begin{bmatrix}
    \mbf{x}_q^\top \widehat{\mbf{Q}}\widehat{\mbf{K}}^\top & \mbf{x}_q^\top \widehat{\mbf{Q}} \mbf{k}
\end{bmatrix} 
\begin{bmatrix}
\frac{1}{n}\sum_{i=1}^{n} \mathbf{x}_{i}\,\mathbf{x}_{i}^{\top} &
 \frac{1}{n}\sum_{i=1}^{n} y_{i}\,\mathbf{x}_{i} \\
\frac{1}{n}\sum_{i=1}^{n} y_{i}\,\mathbf{x}_{i}^{\top} &
\frac{1}{n}\sum_{i=1}^{n} y_{i}^{2}
\end{bmatrix} \mbf{v},
\end{align*}
where we have defined $\widehat{\mbf{Q}} = \begin{bmatrix}
    \mbf{Q} & \mbf{q}_1
\end{bmatrix} \in \mbb{R}^{d\times (d+1)}$, $\widehat{\mbf{K}} = \begin{bmatrix}
    \mbf{K} & \mbf{k}_1
\end{bmatrix} \in \mbb{R}^{d\times (d+1)}$ and $\mbf{k}^\top = \begin{bmatrix}
    \mbf{k}_2^\top & k
\end{bmatrix}$. Consider the following notation:
\begin{align*}
    \widehat{\mbf{\Lambda}} = \frac{1}{n}\sum_{i=1}^{n} \mathbf{x}_{i}\,\mathbf{x}_{i}^{\top}, \quad \mbf{c} = \frac{1}{n}\sum_{i=1}^{n} y_{i}\,\mathbf{x}_{i}, \quad \text{and} \quad \alpha =  \frac{1}{n}\sum_{i=1}^{n} y_{i}^{2}.
\end{align*}
We can then express linear attention as follows:
\begin{align*}
    g_{\mc{W}}\left( \mbf{Z} \right) &= \begin{bmatrix}
        \mbf{x}_q^\top \widehat{\mbf{Q}}\widehat{\mbf{K}}^\top \widehat{\mbf{\Lambda}} +  \mbf{x}_q^\top \widehat{\mbf{Q}}\mbf{k}\mbf{c}^\top & \mbf{x}_q^\top \widehat{\mbf{Q}}\widehat{\mbf{K}}^\top\mbf{c} + \alpha\mbf{x}_q^\top \widehat{\mbf{Q}}\mbf{k}
    \end{bmatrix} \mbf{v} \\
    &= \left( \mbf{x}_q^\top \widehat{\mbf{Q}}\widehat{\mbf{K}}^\top \widehat{\mbf{\Lambda}} +  \mbf{x}_q^\top \widehat{\mbf{Q}}\mbf{k}\mbf{c}^\top\right)\mbf{v}_1 + \left(\mbf{x}_q^\top \widehat{\mbf{Q}}\widehat{\mbf{K}}^\top\mbf{c} + \alpha\mbf{x}_q^\top \widehat{\mbf{Q}}\mbf{k} \right) v.
\end{align*}
Now, recall that we initialize $\mbf{v}_1 = \mbf{0}_d$ and $\mbf{k} = \mbf{0}_{d+1}$. Then, the model simplifies to 
\begin{align*}
    g_{\mc{W}}\left(\mbf{Z}\right) = v\mbf{x}_q^\top \widehat{\mbf{Q}}\widehat{\mbf{K}}^\top \mbf{c}.
\end{align*}
We show that if they are initialized such that $\mbf{v}_1 = \mbf{0}_d$ and $\mbf{k} = \mbf{0}_{d+1}$, they will remain zero: 
    \begin{align*}
        \tau \dot{\mbf{v}}_1 &= \mbb{E} \left[\left(y_q -g_{\mc{W}}\left( \mbf{Z} \right)\right) \cdot \left( \widehat{\mbf{\Lambda}}^\top \widehat{\mbf{K}} \widehat{\mbf{Q}}^\top \mbf{x}_q + \mbf{c} \mbf{k}^\top \widehat{\mbf{Q}}^\top \mbf{x}_q \right) \right] \\
        &= \mbb{E} \left[\left(\mbf{w}^\top \mbf{x}_q - v \mbf{c}^\top \widehat{\mbf{K}} \widehat{\mbf{Q}}^\top \mbf{x}_q \right) \cdot \left( \widehat{\mbf{\Lambda}}^\top \widehat{\mbf{K}} \widehat{\mbf{Q}}^\top \mbf{x}_q \right) \right]  \tag{$\mbf{k} = \mbf{0}_{d+1}$} \\
        &= \mbb{E} \left[\left(\mbf{w}^\top \mbf{x}_q - v \mbf{w}^\top \widehat{\mbf{\Lambda}} \widehat{\mbf{K}} \widehat{\mbf{Q}}^\top \mbf{x}_q \right) \cdot \left( \widehat{\mbf{\Lambda}}^\top \widehat{\mbf{K}} \widehat{\mbf{Q}}^\top \mbf{x}_q \right) \right]  \tag{$\mbf{c}^\top = \frac{1}{n} \sum\limits_{i=1}^n \mbf{w}^\top \mbf{x}_i \mbf{x}_i^\top = \mbf{w}^\top \widehat{\mbf{\Lambda}}$} \\
        &= \mbb{E}[\mbf{w}]^\top \mbb{E}\left[\left(\mbf{x}_q - v\widehat{\mbf{\Lambda}} \widehat{\mbf{K}} \widehat{\mbf{Q}}^\top \mbf{x}_q \right) \cdot \left( \widehat{\mbf{\Lambda}}^\top \widehat{\mbf{K}} \widehat{\mbf{Q}}^\top \mbf{x}_q \right)\right] = \mbf{0}_d. \\
    \tau \dot{\mbf{k}} &= \mbb{E}\left[v\widehat{\mbf{Q}}^\top\mbf{x}_q \cdot \left(\widehat{y}_q -g_{\mc{W}}\left( \mbf{Z} \right)\right) \right] \\
    &= v\widehat{\mbf{Q}}^\top\mbb{E}\left[\mbf{x}_q \cdot \left(\widehat{y}_q -g_{\mc{W}}\left( \mbf{Z} \right)\right) \right] \\
    &= v\widehat{\mbf{Q}}^\top\mbb{E}\left[\mbf{x}_q\right] \cdot\mbb{E} \left[ \left(\widehat{y}_q -g_{\mc{W}}\left( \mbf{Z} \right)\right) \right] \\
    &= \mbf{0}_{d+1}.
    \end{align*}
This gives us the following gradients:
\begin{align*}
    \tau \dot{\widehat{\mbf{Q}}} &= v \mbb{E}\left[ \mbf{x}_q(\widehat{y}_q - v \mbf{x}_q^\top \widehat{\mbf{Q}}\widehat{\mbf{K}}^\top \mbf{c})\mbf{c}^\top\right]  \widehat{\mbf{K}}\\
    \tau \dot{\widehat{\mbf{K}}} &= v \widehat{\mbf{Q}}^\top\mbb{E}\left[ \mbf{x}_q(\widehat{y}_q - v \mbf{x}_q^\top \widehat{\mbf{Q}}\widehat{\mbf{K}}^\top \mbf{c})\mbf{c}^\top\right].
\end{align*}
Note that we also initialize $\mbf{q}_1 = \mbf{k}_1 = \mbf{0}_d$. We show that these also remain zero:
\begin{align*}
    \tau \dot{\mbf{q}}_1 =  \tau \dot{\widehat{\mbf{Q}}} \cdot \mbf{e}_{d+1} &= v \mbb{E}\left[ \mbf{x}_q(\widehat{y}_q - v \mbf{x}_q^\top \widehat{\mbf{Q}}\widehat{\mbf{K}}^\top \mbf{c})\mbf{c}^\top\right]  \widehat{\mbf{K}}  \cdot \mbf{e}_{d+1} \\
    &= v \mbb{E}\left[ \mbf{x}_q(\widehat{y}_q - v \mbf{x}_q^\top \widehat{\mbf{Q}}\widehat{\mbf{K}}^\top \mbf{c})\mbf{c}^\top\right]  \mbf{k}_{1}\\
    &= \mbf{0}.\tag{$\mbf{k}_{1}(0) = \mbf{0}$}
\end{align*}
The proof for $\mbf{k}_1$ follows verbatim. Hence, we have the dynamics
\begin{align*}
    \tau \dot{v} &= \mbb{E}\left[ (\widehat{y}_q - v \mbf{x}_q^\top \mbf{QK}^\top \mbf{c})\cdot \mbf{x}_q^\top \mbf{QK}^\top \mbf{c}  \right]  \\
    &= \mbb{E}\left[ \tr\left((\widehat{y}_q - v \mbf{x}_q^\top \mbf{QK}^\top \mbf{c})\cdot \mbf{x}_q^\top \mbf{QK}^\top \mbf{c} \right) \right] \\
 &= \mbb{E}\left[ \tr\left((\widehat{y}_q - v \mbf{x}_q^\top \mbf{QK}^\top \mbf{c})\cdot \mbf{c}^\top \mbf{K}\mbf{Q}^\top \mbf{x}_q \right) \right] \\
    &= \mbb{E}\left[ \tr\left(\mbf{Q}^\top\mbf{x}_q \cdot (\widehat{y}_q - v \mbf{x}_q^\top \mbf{QK}^\top \mbf{c})\cdot \mbf{c}^\top \mbf{K} \right) \right] \\
      &=  \tr\left(\mbf{Q}^\top\mbb{E}\left[\mbf{x}_q \cdot (\widehat{y}_q - v \mbf{x}_q^\top \mbf{QK}^\top \mbf{c})\cdot \mbf{c}^\top \right]\mbf{K} \right)  \\
    \tau \dot{\mbf{Q}} &= v \mbb{E}\left[ \mbf{x}_q(\widehat{y}_q - v \mbf{x}_q^\top \mbf{QK}^\top \mbf{c})\mbf{c}^\top\right]  \mbf{K} \\
    \tau \dot{\mbf{K}} &= v \mbf{Q}^\top\mbb{E}\left[ \mbf{x}_q(\widehat{y}_q - v \mbf{x}_q^\top \mbf{QK}^\top \mbf{c})\mbf{c}^\top\right].
\end{align*}
Note that 
\begin{align*}
    \mbf{c} = \frac{1}{n} \sum_{i=1}^n y_i \mbf{x}_i = \frac{1}{n} \sum_{i=1}^n \mbf{x}_i \mbf{x}_i^\top \mbf{w} = \widehat{\mbf{\Lambda}} \mbf{w}. 
\end{align*}
With this in mind, let us simplify the common term:
\begin{align*}
    \mbb{E}\left[ \mbf{x}_q(\widehat{y}_q - v \mbf{x}_q^\top \mbf{QK}^\top \mbf{c})\mbf{c}^\top\right] &= \mbb{E}\left[ \mbf{x}_q(\mbf{x}_q^\top\mbf{w} - v \mbf{x}_q^\top \mbf{QK}^\top \mbf{c})\mbf{c}^\top\right] \\
    &= \mbb{E}\left[ \mbf{x}_q\mbf{x}_q^\top \cdot (\mbf{w} - v \mbf{QK}^\top \mbf{c})\mbf{c}^\top\right]  \\
     &= \underbrace{\mbb{E}\left[ \mbf{x}_q\mbf{x}_q^\top \right]}_{= \mbf{I}_d}\cdot\mbb{E}\left[ (\mbf{w} - v \mbf{QK}^\top \mbf{c})\mbf{c}^\top\right] \\
     &= \mbb{E}\left[ (\mbf{I}_d - v \mbf{QK}^\top \widehat{\mbf{\Lambda}}) \cdot \mbf{w}\mbf{w}^\top \widehat{\mbf{\Lambda}}^\top\right] \\
     &= \mbb{E}\left[ \mbf{w} \mbf{w}^\top \widehat{\mbf{\Lambda}} \right] - v \mbf{Q} \mbf{K}^\top \mbb{E}\left[ \widehat{\mbf{\Lambda}} \mbf{w} \mbf{w}^\top \mbf{\Lambda} \right] \\
    &= \mbb{E}\left[ \mbf{w}\mbf{w}^\top \right] \mbb{E}\left[ \widehat{\mbf{\Lambda}} \right] - v \mbf{Q} \mbf{K}^\top \mbb{E}_{\widehat{\mbf{\Lambda}}}\left[ \widehat{\mbf{\Lambda}} \mbb{E}_{\mbf{w}} \left[ \mbf{w} \mbf{w}^\top \right] \widehat{\mbf{\Lambda}} \right] \\
    &= \mbf{\Sigma}_s - v \mbf{Q} \mbf{K}^\top \mbb{E}_{\widehat{\mbf{\Lambda}}} \left[ \widehat{\mbf{\Lambda}} \mbf{\Sigma}_s \widehat{\mbf{\Lambda}} \right].
\end{align*}
Focusing on the $\mbb{E}_{\widehat{\mbf{\Lambda}}} = \left[\widehat{\mbf{\Lambda}}\mbf{\Sigma}_s \widehat{\mbf{\Lambda}} \right]$:
\begin{align*}
    \mbb{E}_{\widehat{\mbf{\Lambda}}} \left[ \widehat{\mbf{\Lambda}} \mbf{\Sigma}_s \widehat{\mbf{\Lambda}} \right] &= \frac{1}{n^2} \mbb{E}\left[ \sum\limits_{i=1}^n \mbf{x}_i \mbf{x}_i^\top \mbf{\Sigma}_s \sum\limits_{j=1}^n \mbf{x}_j \mbf{x}_j^\top  \right] \\
    &= \frac{1}{n^2} \sum\limits_{i=1}^n \sum\limits_{j=1}^n \mbb{E}\left[ \mbf{x}_i \mbf{x}_i^\top \mbf{\Sigma}_s \mbf{x}_j \mbf{x}_j^\top \right] \\
    &= \frac{1}{n^2} \left( \sum\limits_{i=1}^n \sum\limits_{j \neq i} \underbrace{\mbb{E}\left[ \mbf{x}_i \mbf{x}_i^\top \mbf{\Sigma}_s \mbf{x}_j \mbf{x}_j^\top \right]}_{= \mbf{\Sigma}_s} + \sum\limits_{i=1}^n \mbb{E}\left[ \mbf{x}_i \mbf{x}_i^\top \mbf{\Sigma}_s \mbf{x}_i \mbf{x}_i^\top \right] \right) \\
    &= \frac{1}{n^2} \left( n(n - 1) \mbf{\Sigma}_s + \sum\limits_{i=1}^n \left( 2 \mbf{\Sigma_s} + \operatorname{Tr}\left( \mbf{\Sigma}_s \right) \mbf{I}_d  \right) \right) \\
    &= \frac{n + 1}{n} \mbf{\Sigma}_s + \frac{\operatorname{Tr}(\mbf{\Sigma}_s)}{n} \mbf{I}_d.
\end{align*}
Therefore,
\begin{align*}
    &\mbb{E}\left[ \mbf{x}_q (y_q - v \mbf{x}_q^\top \mbf{Q} \mbf{K}^\top \mbf{c}) \mbf{c}^\top \right] = \mbf{\Sigma}_s - v \mbf{Q} \mbf{K}^\top \left( \frac{n + 1}{n} \mbf{\Sigma}_s + \frac{\operatorname{Tr}(\mbf{\Sigma}_s)}{n} \mbf{I}_d \right), 
\end{align*}
which imply the dynamics:
\begin{align*}
    &\tau \dot{v} = \operatorname{Tr} \left( \mbf{Q}^\top \left( \mbf{\Sigma}_s - v \mbf{Q} \mbf{K}^\top \left( \frac{n + 1}{n} \mbf{\Sigma}_s + \frac{\operatorname{Tr}(\mbf{\Sigma}_s)}{n} \mbf{I}_d \right)  \right) \mbf{K} \right), \\
    &\tau\dot{\mbf{Q}} = v \left( \mbf{\Sigma}_s - v \mbf{Q} \mbf{K}^\top \left( \frac{n + 1}{n} \mbf{\Sigma}_s + \frac{\operatorname{Tr}(\mbf{\Sigma}_s)}{n} \mbf{I}_d \right) \right) \mbf{K}, \\
    &\tau \dot{\mbf{K}} = v \mbf{Q}^\top \left(  \mbf{\Sigma}_s - v \mbf{Q} \mbf{K}^\top \left( \frac{n + 1}{n} \mbf{\Sigma}_s + \frac{\operatorname{Tr}(\mbf{\Sigma}_s)}{n} \mbf{I}_d \right) \right).
\end{align*}
By taking the limit $n\to \infty$, we obtain the following:
This gives us the dynamics
\begin{align*}
    \tau \underline{\dot{v}} \coloneqq\tau \lim_{n\to \infty} \dot{v} &= \tr\left(\mbf{Q}^\top \left(\mbf{\Sigma}_s - v \mbf{QK}^\top \mbf{\Sigma}_s \right) \mbf{K}\right), \\
    \tau \underline{\dot{\mbf{Q}}} \coloneqq\tau \lim_{n\to \infty} \dot{\mbf{Q}} &= v \left(\mbf{\Sigma}_s - v \mbf{QK}^\top \mbf{\Sigma}_s \right)  \mbf{K}, \\
    \tau \underline{\dot{\mbf{K}}} \coloneqq\tau \lim_{n\to \infty} \dot{\mbf{K}} &= v \mbf{Q}^\top\left(\mbf{\Sigma}_s - v \mbf{QK}^\top \mbf{\Sigma}_s \right).
\end{align*}
Now it remains to show that the updates of $\mbf{Q} \in \mbb{R}^{d\times d}$ and $\mbf{K} \in \mbb{R}^{d\times d}$ occur in invariant subspaces. Let $\mbf{\Sigma}_s = \mbf{U}_s \mbf{\Lambda}_s \mbf{U}_s^\top$ denote the eigendecomposition of the task covariance matrix $\mbf{\Sigma}_s \in \mbb{R}^{d\times d}$. Consider the following change of variables:
\begin{align*}
    \mbf{\Lambda}_Q = \mbf{U}_s^\top \mbf{Q} \mbf{U}_s \quad \text{and} \quad \mbf{\Lambda}_K = \mbf{U}_s^\top \mbf{K} \mbf{U}_s
\end{align*}
By plugging in the dynamics for $\dot{\underline{\mbf{Q}}}$ and $\dot{\underline{\mbf{K}}}$, we obtain
\begin{align*}
    \tau \underline{\dot{\mbf{\Lambda}}}_Q &= v \left(\mbf{\Lambda}_s - v \mbf{\Lambda}_{Q} \mbf{\Lambda}_{K} \mbf{\Lambda}_s \right) \mbf{\Lambda}_{K} \\ 
     \tau \underline{\dot{\mbf{\Lambda}}}_K &= v\mbf{\Lambda}_{Q} \left(\mbf{\Lambda}_s - v \mbf{\Lambda}_{Q} \mbf{\Lambda}_{K} \mbf{\Lambda}_s \right) 
\end{align*}
By initialization, note that
\begin{align*}
    \mbf{Q}(0) = \alpha \mbf{I}_d = \alpha \mbf{U}_s \mbf{U}_s^\top \implies \mbf{\Lambda}_Q(0) = \alpha\mbf{I}_d \quad \text{and} \quad \mbf{K}(0) = \alpha \mbf{I}_d = \alpha \mbf{U}_s \mbf{U}_s^\top \implies \mbf{\Lambda}_K(0) = \alpha\mbf{I}_d.
\end{align*}
Since they are diagonal matrices at initialization, the GF dynamics are also diagonal, and remain diagonal for all $t\geq 0$. Returning to the original coordinates, this implies that both $\mbf{Q}(t)$ and $\mbf{K}(t)$ are jointly diagonalizable in the $\mbf{\Sigma}_s$ basis for all $t\geq 0$, and so we have the following dynamics:
\begin{align*}
    \tau \underline{\dot{\mbf{Q}}} \coloneqq\tau \lim_{n\to \infty} \dot{\mbf{Q}} &= \tau \lim_{n\to \infty} \mbf{U}_s \dot{\mbf{\Lambda}}_Q \mbf{U}_s^\top =  \mbf{U}_s \left(v\left(\mbf{\Lambda}_s - v \mbf{\Lambda}_{Q} \mbf{\Lambda}_{K} \mbf{\Lambda}_s \right) \mbf{\Lambda}_{K} \right) \mbf{U}_s^\top,\\ 
    \tau\underline{\dot{\mbf{K}}} \coloneqq \tau \lim_{n\to \infty} \dot{\mbf{K}} &= \tau \lim_{n\to \infty} \mbf{U}_s \dot{\mbf{\Lambda}}_K \mbf{U}_s^\top =  \mbf{U}_s\left(v\mbf{\Lambda}_{Q} \left(\mbf{\Lambda}_s - v \mbf{\Lambda}_{Q} \mbf{\Lambda}_{K} \mbf{\Lambda}_s \right)\right) \mbf{U}_s^\top, \\
    \tau\underline{\dot{v}} \coloneqq \tau \lim_{n\to \infty}\dot{v} &= \tr\left(\mbf{\Lambda}_{Q} \left( \mbf{\Lambda}_s - v \mbf{\Lambda}_{Q} \mbf{\Lambda}_{K} \mbf{\Lambda}_s\right)\mbf{\Lambda}_{K}\right).
\end{align*}
This completes the proof.
\end{proof}

\subsection{Proof of \Cref{lemma:conservation}}

\begin{proof}
It suffices to show that $\frac{d}{dt}\left[\sum_{i=1}^q \lambda_i^2(t) - v^2(t)\right] = 0$. Recall from the gradient flow dynamics that $\tau\underline{\dot{v}} = \sum_{i=1}^d \sigma_i(1 - v\lambda_i^2)\lambda_i^2$. Because the task covariance matrices are exactly rank $q$, the signal strength $\sigma_i = 0$ for all $i > q$. Therefore, the tail sum vanishes, and we can write the derivative of the invariant as:
\begin{align*}
\tau \frac{d}{dt}\left[\sum_{i=1}^q \lambda_i^2 - v^2\right] 
&= 2\sum_{i=1}^q \lambda_i \cdot \tau\underline{\dot{\lambda}}_i - 2v \cdot \tau\underline{\dot{v}} \\
&= 2\sum_{i=1}^q \lambda_i \cdot \sigma_i(1 - v\lambda_i^2) v\lambda_i - 2v \sum_{i=1}^q \sigma_i(1 - v\lambda_i^2) \lambda_i^2 \\
&= 2v \sum_{i=1}^q \sigma_i(1 - v\lambda_i^2)\lambda_i^2 - 2v \sum_{i=1}^q \sigma_i(1 - v\lambda_i^2)\lambda_i^2 = 0.
\end{align*}
Since the derivative is zero, the quantity $\sum_{i=1}^q \lambda_i^2(t) - v^2(t)$ is constant for all $t \geq 0$. By the initial conditions in \Cref{eqn:gf_initial}, $\sum_{i=1}^q \lambda_i^2(0) - v^2(0) = 0$, which completes the proof.
\end{proof}

\subsection{Derivation for Single Task Case}
\label{sec:sing_subspace_time}

\begin{proof}
    For the case in which the training task vectors are drawn from a single subspace (i.e., $\nin = 0$), we need to compute
\begin{align*}
    t_{\text{single}} = \tau q \int_{\alpha}^{\lambda^\star - \delta} \frac{d\lambda}{(1-v\lambda^2) \cdot v\lambda}  = \underbrace{\tau q \int_{\alpha}^{\lambda^\star - \delta} \frac{d\lambda}{v\lambda}}_{= t_{\textup{single, growth}}} + \underbrace{\tau q \int_{\alpha}^{\lambda^\star - \delta} \frac{\lambda}{1 - v\lambda^2} d\lambda}_{= t_{\textup{single, sat}}},
\end{align*}
as stated in \Cref{eqn:t_sing_decomp}. By \Cref{lemma:conservation}, the conservation law gives us
\begin{align*}
    v^2(t) = q\lambda^2(t) \implies v(t) = \sqrt{q} \lambda(t).
\end{align*}
By plugging this into $t_{\text{single}}$, we obtain
\begin{align*}
    t_{\text{single}} &= \tau \sqrt{q} \int_{\alpha}^{\lambda^\star - \delta} \frac{d\lambda}{\lambda^2} + \tau q \int_{\alpha}^{\lambda^\star - \delta} \frac{\lambda}{1 - \sqrt{q}\lambda^3} d\lambda \\
    &= \tau \sqrt{q} \left(\frac{1}{\alpha} - \frac{1}{\lambda^\star - \delta} \right)+ \tau q \int_{\alpha}^{\lambda^\star - \delta} \frac{\lambda}{1 - \sqrt{q}\lambda^3} d\lambda.
\end{align*}
We focus on the saturation term: it is straightforward to show that
\begin{align*}
    h(x) \coloneqq -\frac{1}{3} \log |x - 1| + \frac{1}{6} \log(x^2 + x + 1) - \frac{\sqrt{3}}{3} \arctan\left(\frac{2x + 1}{\sqrt{3}}\right) \implies h'(x) = \frac{x}{1-x^3},
\end{align*}
and so 
\begin{align*}
    \tau q \int_{\alpha}^{\lambda^\star - \delta} \frac{\lambda}{1 - v\lambda^2} d\lambda = \tau q \cdot q^{-1/3} \left[h\left(q^{1/6}\left(\lambda^\star - \delta \right) \right) - h\left(q^{1/6}\alpha \right)\right].
\end{align*}
Recall that in order for the system to be stationary, we require $v\lambda^{\star 2} = 1$, and with the conservation law, this gives us $\lambda^\star = q^{-1/6}$. Plugging this in, we obtain the following:
\begin{align*}
    t_{\text{single}} = \tau \sqrt{q} \left(\frac{1}{\alpha} - \frac{1}{\lambda^\star - \delta} \right)+ \tau q^{2/3} \cdot \left(h\left(1 - q^{1/6}\delta \right) - h\left(q^{1/6}\alpha \right)\right),
\end{align*}
which is exactly the expression in \Cref{eqn:t_sing_total}. This completes the derivation.
\end{proof}

\subsection{Derivation for Task Diverse Case}
\label{sec:mix_subspace_time}

\begin{proof}
Recall that our GF dynamics are given by
\begin{align*}
\tau \underline{\dot{\lambda}}_{\text{over}} &= \frac{1}{a} \left( 1 - v\lov^2 \right) v\lov \\
\tau \underline{\dot{\lambda}}_{\text{ind}} &= \frac{1}{2a} \left( 1 - v\lin^2 \right) v\lin \\
\tau \underline{\dot{v}} &= \frac{\nov}{a} \left( 1-v\lov^2  \right) v\lov + \frac{\nin}{2a} \left( 1-v\lin^2  \right) v\lin.
\end{align*}
Our goal is to derive $t_{\text{mix}}$, which can be decomposed into $t_{\text{over}} + t_{\text{ind}}$ under the power-law assumption, which states that $\lambda_{\text{ind}}$ can be approximated using $\lambda_{\text{ind}} = \sqrt{\alpha \lambda_{\text{over}}}$ during $t_{\text{over}}$. Then, at $t_{\text{ind}}$, we solve for the remaining time it takes for $\lambda_{\text{over}}$ to reach stationarity.

\paragraph{Stage I: Learning the Common Subspace Components.}

Following the derivation strategy in \Cref{sec:sing_subspace_time}, we have that 
\begin{align}
\label{eqn:t_over_init}
    t_{\text{over}} = a\tau \int_{\alpha}^{\lov^\star - \delta} \frac{d\lov}{\left( 1 - v\lov^2\right) v\lov},
\end{align}
which admits the following decomposition:
\begin{align}
    t_{\text{over}} = a\tau \int_{\alpha}^{\lov^\star - \delta} \frac{d\lov}{\left( 1 - v\lov^2\right) v\lov} = \underbrace{a\tau \int_{\alpha}^{\lov^\star - \delta} \frac{1}{v\lov}}_{= t_{\textup{over, growth}}} + \underbrace{a\tau \int_{\alpha}^{\lov^\star - \delta} \frac{v\lov^2}{1-v\lov^2}\cdot \frac{1}{v\lov}}_{= t_{\textup{over, sat}}}.
\end{align}
For this setting, we have the following conservation law by \Cref{lemma:conservation}:
\begin{align*}
    v^2(t) = \nov \lov^2(t) + \nin \lin^2(t).
\end{align*}
However, due to the coupled nature of $\lin$ and $\lov$, we cannot straightforwardly plug in the conservation law to derive the time as done in \Cref{sec:sing_subspace_time}. Instead, note that under the power-law assumption, $\lambda_{\text{ind}} = \sqrt{\alpha \lambda_{\text{over}}}$, and so 
\begin{align*}
    v^2(t) = \nov \lov^2 + \alpha \nin \lov.
\end{align*}
By plugging this into $t_{\textup{over, growth}}$, we obtain 
\begin{align*}
    t_{\textup{over, growth}} &= \tau a \int_{\alpha}^{\lov^{\star}-\delta} \frac{d\lov}{\lov^{3/2} \sqrt{\nov \lov + \alpha\nin}} \\ &= \frac{2\tau a}{\alpha \nin} \left(\frac{\sqrt{\nin \alpha + \nov \alpha}}{\sqrt{\alpha}} - \frac{\sqrt{\nin \alpha + \nov (\lov^\star - \delta)}}{\sqrt{\lov^\star - \delta}} \right) \\
    &=\frac{2\tau a}{\alpha \nin} \left(\sqrt{q} - \frac{\sqrt{\nin \alpha + \nov (\lov^\star - \delta)}}{\sqrt{\lov^\star - \delta}} \right).
\end{align*}
The challenging part is the saturation time, which does not admit a closed-form integral, even with the (approximate) conservation law. To handle this, we can make an ansatz on the GF dynamics to simplify into a 1D system. We highlight that this is  commonly done in the literature \cite{zhangtraining, varre2023on}. Let us define $u_{\text{over}} \coloneqq v \lov^2$, which gives us 
\begin{align*}
    t_{\textup{over, sat}} = a\tau \int_{\alpha}^{\lov^\star - \delta} \frac{v\lov^2}{1-v\lov^2}\cdot \frac{1}{v\lov} = a\tau \int_{\alpha}^{\lov^\star - \delta} \frac{u_{\text{over}}}{1-u_{\text{over}}}\cdot \frac{1}{v\lov}
\end{align*}
By the conservation law:
\begin{align*}
    v^2(t) = \nov \lov^2 + \alpha \nin \lov \implies u_{\text{over}}(t) = \lov^2\sqrt{\nov \lov^2 + \alpha \nin \lov},
\end{align*}
which implies that $u_{\text{over}} \propto \sqrt{\nov }\lov^3$  for sufficiently small $\alpha$. 
However, using the ansatz that $u_{\text{over}} = \sqrt{\nov }\lov^3$ may be inaccurate, as $u_{\text{over}} \neq 1$ when $\lov = \lov^\star$. To this end, we make the ansatz that $u_{\text{over}}$ can be approximated using the following calibrated cubic:
\begin{align*}
    u_{\text{over}} = \left(\frac{\lov}{\lov^\star}\right)^3.
\end{align*}
Then we have 
\begin{align*}
    \frac{u_{\text{over}}}{1-u_{\text{over}}}\cdot \frac{1}{v\lov} = \frac{u_{\text{over}}}{1-u_{\text{over}}}\cdot \frac{\lov}{u_{\text{over}}} = \frac{\lov}{1 - (\lov / \lov^\star)^3}.
\end{align*}
This gives us the following saturation time:
\begin{align*}
    t_{\textup{over, sat}} = a\tau \int_{\alpha}^{\lov^\star - \delta} \frac{\lov}{1 - (\lov / \lov^\star)^3} d\lov = a \tau \lov^{\star 2} \left(h\left( \frac{\lov^\star - \delta}{\lov^\star}\right) - h\left(\frac{\alpha}{\lov^\star} \right)\right),
\end{align*}
where 
\begin{align*}
    h(x) = -\frac{1}{3} \log |x - 1| + \frac{1}{6} \log(x^2 + x + 1) - \frac{\sqrt{3}}{3} \arctan\left(\frac{2x + 1}{\sqrt{3}}\right).
\end{align*}
is the same function as seen in \Cref{sec:sing_subspace_time}. Furthermore, note that if $\lov^\star = q^{-1/6}$ as in \Cref{sec:sing_subspace_time}, the two saturation times are equivalent, corroborating the use of the ansatz. Now it remains to compute $\lov^\star$. 
By plugging in the approximate conservation law into the dynamics of $\lov$, the stationarity condition for $\lov$ gives us that that optimal solution $\lov^\star$ should satisfy
\begin{align*}
    v\lov^2 = 1 &\implies \lov^2\sqrt{\nov \lov + \nin \alpha^{1/2} \lov}  = 1 \\
    &\implies \nov \lov^{\star 6} + \alpha \nin \lov^{\star 5} - 1 = 0.
\end{align*}
Putting it all together, we have that the time for Stage I is given by:
\begin{align*}
    t_{\text{over}}  
    = \underbrace{\frac{2\tau a}{\alpha \nin} \left(\frac{\sqrt{\nin \alpha + \nov \alpha}}{\sqrt{\alpha}} - \frac{\sqrt{\nin \alpha + \nov (\lov^\star - \delta)}}{\sqrt{\lov^\star - \delta}} \right)}_{t_{\text{over, growth}}} + \underbrace{a \tau \lov^{\star 2} \left(h\left( \frac{\lov^\star - \delta}{\lov^\star}\right) - h\left(\frac{\alpha}{\lov^\star} \right)\right)}_{t_{\text{over, sat}}}
\end{align*}
where $\lov^\star$ is the positive root of 
\begin{align*}
     \nov \lov^{*6} + \alpha \nin \lov^{*5} - 1 = 0.
\end{align*}
This completes the derivation.

\paragraph{Stage II: Learning the Independent Subspace Components.}

Once the shared subspace has been learned and saturated, we are left with the dynamics in terms of $\lin$ and $v$. We can set $u_{\text{over}} = v\lov^{\star 2} = 1$, which also gives us the conservation law
\begin{align*}
    v^2(t) = \nov \lov^{\star 2}(t) + \nin \lin^2(t).
\end{align*}
Now, we can take the latter approach  in \Cref{sec:sing_subspace_time}, and write the dynamics in terms of $\dot{u}_{\text{ind}} \coloneqq v\lin^2$:
\begin{align*}
    \tau \dot{u}_{\text{ind}} &= n_{\text{ind}}   \frac{u^2_{\text{ind}}(1 - u_{\text{ind}})}{(n_{\text{over}} + n_{\text{ind}} u_{\text{ind}})^{2/3}} + (n_{\text{over}} + n_{\text{ind}} u_{\text{ind}})^{1/3} (1 - u_{\text{ind}}) u_{\text{ind}} \\
    &= \frac{u_{\text{ind}}(1 - u_{\text{ind}})(n_{\text{over}} + 2n_{\text{ind}} u_{\text{ind}})}{ \left( n_{\text{over}} + n_{\text{ind}} u_{\text{ind}}\right)^{2/3} }
\end{align*}
The time it takes for $u_{\text{ind}}$ to reach $1-\delta$ from its initial value $u_{\text{ind}}(0)>0$  involves solving the integral
\begin{align*}
    t_{\text{ind}} = \tau a \int_{u_{\text{ind}}(0)}^{1-\delta} \frac{\left( n_{\text{over}} + n_{\text{ind}} u_{\text{ind}}\right)^{2/3}}{u_{\text{ind}}(1 - u_{\text{ind}})(n_{\text{over}} + 2n_{\text{ind}} u_{\text{ind}})}
\end{align*}
Unfortunately, this integral does not have a closed-form solution. Instead, we can derive upper and lower bounds on $t_{\text{ind}}$. Consider the function
\begin{align*}
    g(u_{\text{ind}}) = \frac{\left( n_{\text{over}} + n_{\text{ind}} u_{\text{ind}}\right)^{2/3}}{n_{\text{over}} + 2n_{\text{ind}} u_{\text{ind}}}.
\end{align*}
We are interested in the case where $u_{\text{ind}} \in (0, 1)$. On this interval, notice that $g(u_{\text{ind}})$ is decreasing:
\begin{align*}
    g'(\ui) = - \frac{2\nin (2\nov + \nin \ui)}{3(\nov + \nin \ui)^{1/3}(\nov + 2\nin \ui)^2} < 0 \implies g(1-\delta)\leq g(\ui) \leq g(u_{\text{ind}}(0)).
\end{align*}
This gives us the following bounds:
\begin{align*}
     g(1-\delta) \int_{u_{\text{ind}}(0)}^{1-\delta} \frac{d\ui}{\ui (1-\ui)} &\leq \, \frac{t_{\text{ind}} }{\tau a} \,\leq  g(u_{\text{ind}}(0)) \int_{u_{\text{ind}}(0)}^{1-\delta} \frac{d\ui}{\ui (1-\ui)} \\
     \tau a \cdot \log\left( \frac{(1-\delta)(1 - u_{\text{ind}}(0))}{\delta u_{\text{ind}}(0)}\right)g(1-\delta)  &\leq \, t_{\text{ind}} \, \leq \tau a \cdot \log\left( \frac{(1-\delta)(1 - u_{\text{ind}}(0))}{\delta u_{\text{ind}}(0)}\right)g(u_{\text{ind}}(0)).
\end{align*}
It remains to plug in a value for $u_{\text{ind}}(0)$. With the power law and the fact that $v\lov^{\star 2} = 1$ at the end of Stage I, we have 
\begin{align*}
    u_{\text{ind}}(0) = \left(1/ \lov^{\star 2}\right) \cdot \alpha \lov^\star = \alpha / \lov^\star.
\end{align*} 
This completes the derivation.
    
\end{proof}

\subsection{Derivation for \Cref{res:strict_acceleration}}
\label{sec:strict_acc_proof}

\begin{proof}
    Recall that the primary bottleneck in $t_{\textup{mix}}$ is $t_{\textup{over}}$, i.e., the time it takes to learn the common subspace components. We first show that the saturation terms, namely $t_{\textup{single,sat}}$ and $t_{\textup{over,sat}}$, are subleading. Indeed, $h$ is finite at $x = 0$:
\begin{align*}
    h(0)
    =
    -\frac{\sqrt{3}}{3}\arctan\!\left(\frac{1}{\sqrt{3}}\right)
    =
    -\frac{\pi\sqrt{3}}{18}
    =
    \mc{O}(1),
\end{align*}
and has only a logarithmic singularity as $x \to 1$. Consequently, both $t_{\textup{single,sat}}$ and $t_{\textup{over,sat}}$ scale as $\mc{O}(\log(1/\delta))$. In contrast, their growth counterparts diverge hyperbolically as $\alpha \to 0$. Therefore, for fixed $\delta$ and sufficiently small $\alpha$, the growth time is the dominant bottleneck in both the single-task and task-diverse settings.

We now derive the leading-order growth-time approximations. From \Cref{eqn:main_sing_decomp}, the single-task growth time is
\begin{align*}
    t_{\textup{single,growth}}
    =
    \tau \sqrt{q}
    \left(
        \frac{1}{\alpha}
        -
        \frac{1}{q^{-1/6}-\delta}
    \right).
\end{align*}
Since $\delta \in (0,q^{-1/6})$ is fixed, the second term is independent of $\alpha$. Therefore, as $\alpha \to 0$,
\begin{align*}
    t_{\textup{single,growth}}
    =
    \frac{\tau \sqrt{q}}{\alpha}
    -
    \frac{\tau \sqrt{q}}{q^{-1/6}-\delta}
    \approx
    \frac{\tau \sqrt{q}}{\alpha}.
\end{align*}
Similarly, from \Cref{eqn:main_mix_decomp}, the shared-stage growth time is
\begin{align*}
    t_{\textup{over,growth}}
    =
    \frac{2\tau a}{\alpha\nin}
    \left(
        \sqrt{q}
        -
        \sqrt{\nov + \frac{\nin \alpha}{\lov^\star-\delta}}
    \right), \quad \text{where} \quad a \coloneqq \frac{2\nov+\nin}{2}.
\end{align*}
Since $q = \nov+\nin$, we have
\begin{align*}
    2a
    =
    2\nov+\nin
    =
    q+\nov.
\end{align*}
Thus,
\begin{align*}
    \frac{2\tau a}{\alpha\nin}
    =
    \frac{\tau(q+\nov)}{\alpha\nin}.
\end{align*}
For sufficiently small $\alpha$, we have $\lov^\star \approx \nov^{-1/6}$ from \Cref{eqn:main_lambda_star_eqn}. Since $\delta < q^{-1/6} \leq \nov^{-1/6}$, the denominator $\lov^\star-\delta$ remains positive and bounded away from zero as $\alpha \to 0$. Therefore,
\begin{align*}
    \frac{\nin\alpha}{\lov^\star-\delta}
    =
    \mc{O}(\alpha),
\end{align*}
and hence
\begin{align*}
    \sqrt{\nov + \frac{\nin\alpha}{\lov^\star-\delta}}
    \approx
    \sqrt{\nov}.
\end{align*}
It follows that
\begin{align*}
    t_{\textup{over,growth}}
    &\approx
    \frac{\tau(q+\nov)}{\alpha\nin}
    \left(
        \sqrt{q}
        -
        \sqrt{\nov}
    \right).
\end{align*}
Using $\nin = q-\nov$, we simplify
\begin{align*}
    \frac{\sqrt{q}-\sqrt{\nov}}{\nin}
    =
    \frac{\sqrt{q}-\sqrt{\nov}}{q-\nov}
    =
    \frac{1}{\sqrt{q}+\sqrt{\nov}},
\end{align*}
where the last equality follows from
\begin{align*}
    q-\nov
    =
    \left(\sqrt{q}-\sqrt{\nov}\right)
    \left(\sqrt{q}+\sqrt{\nov}\right).
\end{align*}
Therefore,
\begin{align*}
    t_{\textup{over,growth}}
    \approx
    \frac{\tau(q+\nov)}{\alpha(\sqrt{q}+\sqrt{\nov})}.
\end{align*}
Combining the two leading-order approximations gives
\begin{align*}
    \frac{t_{\textup{over,growth}}}{t_{\textup{single,growth}}}
    \approx
    \frac{q+\nov}{\sqrt{q}(\sqrt{q}+\sqrt{\nov})} = \frac{q+\nov}{q+\sqrt{q\nov}} =
    \frac{1+\nov/q}{1+\sqrt{\nov/q}}.
\end{align*}
Recalling that $\kappa \coloneqq \nov/q$, we obtain
\begin{align*}
    \frac{t_{\textup{over,growth}}}{t_{\textup{single,growth}}}
    \approx
    \rho(\kappa)
    \coloneqq
    \frac{1+\kappa}{1+\sqrt{\kappa}}.
\end{align*}

Now, it remains to verify the claimed properties of $\rho$. First, for every $\kappa \in (0,1)$,
\begin{align*}
    \rho(\kappa) < 1
    \quad
    &\Longleftrightarrow
    \quad
    1+\kappa < 1+\sqrt{\kappa} \\
    &\Longleftrightarrow
    \quad
    \kappa < \sqrt{\kappa},
\end{align*}
which holds exactly when $\kappa \in (0,1)$. Hence, task diversity gives a strict reduction in the leading growth time whenever the two task subspaces have both shared and non-overlapping components.
To find the optimal overlap ratio, set $s = \sqrt{\kappa} \in [0,1]$. Then
\begin{align*}
    \rho(\kappa)
    =
    \frac{1+\kappa}{1+\sqrt{\kappa}}
    =
    \frac{1+s^2}{1+s}
    \eqqcolon
    \widetilde{\rho}(s).
\end{align*}
Differentiating with respect to $s$ gives
\begin{align*}
    \widetilde{\rho}'(s)
    =
    \frac{2s(1+s) - (1+s^2)}{(1+s)^2}
    =
    \frac{s^2 + 2s - 1}{(1+s)^2}.
\end{align*}
Thus, the unique critical point in $[0,1]$ satisfies
\begin{align*}
    s^2 + 2s - 1 = 0.
\end{align*}
Solving this quadratic gives
\begin{align*}
    s^\star = \sqrt{2} - 1.
\end{align*}
Therefore,
\begin{align*}
    \kappa^\star
    =
    (s^\star)^2
    =
    (\sqrt{2}-1)^2
    =
    3 - 2\sqrt{2}.
\end{align*}
At this value,
\begin{align*}
    \rho(\kappa^\star)
    =
    \frac{1+(\sqrt{2}-1)^2}{1+(\sqrt{2}-1)} =
    \frac{4 - 2\sqrt{2}}{\sqrt{2}} =
    2(\sqrt{2}-1).
\end{align*}
Since $2(\sqrt{2}-1) \approx 0.828$, the optimal overlap ratio yields roughly a $17\%$ reduction in the leading growth time. This completes the derivation.
\end{proof}

\subsection{Comparing Convergence Times Directly}
\label{sec:compare_convergence}

In this section, rather than using approximations for $t_{\text{single, growth}}$ and $t_{\text{over, growth}}$ as done in \Cref{sec:strict_acc_proof}, we prove the result directly, which reduces to showing that
\begin{align*}
    \underbrace{\frac{2\tau a}{\alpha \nin} \left(\sqrt{q} - \frac{\sqrt{\nin \alpha + \nov (\lov^\star - \delta)}}{\sqrt{\lov^\star - \delta}} \right)}_{\eqqcolon f(\lov^\star)} < \tau \sqrt{q} \left(\frac{1}{\alpha} - \frac{1}{q^{-1/6} - \delta}\right).
\end{align*}
We first aim to upper bound $f(\lov^\star)$. Note that
\begin{align*}
    f'(\lov^\star) = \frac{2\tau a}{2(\lov^\star - \delta)^2 \sqrt{\nov + \frac{\nin \alpha}{\lov^\star - \delta}}} > 0,
\end{align*}
and so $f$ is increasing for $\lov^\star > \delta$. Recall that  $\lov^\star$ is the positive root of the polynomial
\begin{align*}
    g(\lov^\star) \coloneqq \nov \lov^{*6} + \alpha \nin \lov^{*5} - 1.
\end{align*}
It is straightforward to show that $g(\cdot)$ is also an increasing function, and so by monotonicity:
\begin{align*}
    g(\nov^{-1/6}) = \alpha \nov^{-5/6} \geq 0 \implies \lov^\star \leq \nov^{-1/6}.
\end{align*}
By plugging in this upper bound to $f(\cdot)$, it suffices to show that
\begin{align*}
     \frac{2 a}{\alpha \nin} \left(\sqrt{q} - \frac{\sqrt{\nin \alpha + \nov (\nov^{-1/6} - \delta)}}{\sqrt{\nov^{-1/6} - \delta}} \right) <  \sqrt{q} \left(\frac{1}{\alpha} - \frac{1}{q^{-1/6} - \delta}\right),
\end{align*}
where we have the following definitions:
\begin{align*}
    a = \frac{2\nov + \nin}{2}
    \quad \text{and} \quad
    q = \nin + \nov.
\end{align*}
To show this relation, let us first define the following constants:
\begin{align*}
    B=n_{\mathrm{over}}^{-1/6}-\delta \quad  \text{and} \quad D=q^{-1/6}-\delta.
\end{align*}
Then, we focus on simplifying the left-hand side:
\begin{align*}
\frac{\sqrt{n_{\mathrm{ind}}\alpha+n_{\mathrm{over}}B}}{\sqrt{B}}
=
\sqrt{n_{\mathrm{over}}+\frac{n_{\mathrm{ind}}\alpha}{B}},
\end{align*}
and so we obtain
\begin{align*}
    \frac{2a}{\alpha \nin} \cdot \left(\sqrt{q} -  \sqrt{n_{\mathrm{over}}+\frac{n_{\mathrm{ind}}\alpha}{B}}\right) &< \sqrt{q} \left( \frac{1}{\alpha} - \frac{1}{D}\right) \\
    \frac{2a}{ \nin} \cdot \left(\sqrt{q} -  \sqrt{n_{\mathrm{over}}+\frac{n_{\mathrm{ind}}\alpha}{B}}\right) &< \sqrt{q} \left( 1 - \frac{\alpha}{D}\right) \tag{Multiply by $\alpha > 0$} \\
    K\sqrt{q} - K \sqrt{\nov + \frac{\nin\alpha}{B}} &< \sqrt{q} - \frac{\alpha \sqrt{q}}{D}\tag{$K\coloneqq \frac{2\nov + \nin}{\nin}$} \\
    K\sqrt{\nov + \frac{\nin \alpha}{B}} &> (K-1)\sqrt{q} + \frac{\alpha\sqrt{q}}{D} \\
    \sqrt{\nov + \frac{\nin \alpha}{B}} &> \frac{K-1}{K}\sqrt{q} + \frac{1}{KD}\alpha\sqrt{q} \tag{Divide by $K>0$} \\
    \sqrt{\nov + \frac{\nin \alpha}{B}} &> \frac{\nov \sqrt{q}}{a} + \frac{\alpha \nin \sqrt{q}}{2aD} \\
    \nov + \frac{\nin \alpha}{B} &> \left(\frac{\nov \sqrt{q}}{a} + \frac{\alpha \nin \sqrt{q}}{2aD}\right)^2.
\end{align*}
Bringing everything to one side, we have that 
\begin{align}
    \frac{q\,n_{\mathrm{ind}}^{2}}{4a^{2}D^{2}}\,\alpha^{2}
+
n_{\mathrm{ind}}\!\left(\frac{q\,n_{\mathrm{over}}}{a^{2}D}-\frac{1}{B}\right)\alpha
+
\left(\frac{q\,n_{\mathrm{over}}^{2}}{a^{2}}-n_{\mathrm{over}}\right)
<0,
\label{eqn:inequality_condition}
\end{align}
which shows that the condition for this to hold is quadratic in $\alpha$. Notice that 
\begin{align*}
\frac{q\,n_{\mathrm{over}}^{2}}{a^{2}}-n_{\mathrm{over}}
=
-\frac{n_{\mathrm{over}}\,n_{\mathrm{ind}}^{2}}{4a^{2}}
<0.
\end{align*}
Then, the parabola in \Cref{eqn:inequality_condition} opens upward after $\alpha > 0$, implying that there exists one positive root, say $\alpha'$ such that for all $0 < \alpha < \alpha'$, the condition in \Cref{eqn:inequality_condition} holds. In other words, for a sufficiently small initialization scale $\alpha > 0$, task diversity will have a faster growth time, thereby accelerating convergence.

\section{Deferred Proofs from \Cref{sec:ood}}

\subsection{Proof of \Cref{thm:pos_result} and \Cref{thm:mixture_k_subspaces}}

\begin{proof}

We only provide a proof for \Cref{thm:mixture_k_subspaces}, as \Cref{thm:pos_result} is a special case of \Cref{thm:mixture_k_subspaces} when $K = 2$, $\alpha_1 = \sin(\theta)$, and $\alpha_2 = \cos(\theta)$ for any $\theta \in [0, \pi/2]$. 

\bigskip

\noindent For simplicity, let $\widetilde{y} := \widetilde{y}_{m+1}$. By \cite[Lemma 1 and Lemma 4]{kwon2026outofdistribution}, we have
    \begin{align}
        \mbb{E} &\left[ \left(\widetilde{y} - g^\star_{\mc{W}}(\widetilde{\mbf{Z}}) \right)^2 \right] = \Big( \frac{1}{m} \operatorname{Tr}\left( \mbf{A}^\top\mbf{A} \right) + 1 \Big) \Big(\operatorname{Tr}\left(\overline{\mbf{\Sigma}}_t\right) + \sigma^2 \Big) - 2 \operatorname{Tr}\left(\overline{\mbf{\Sigma}}_t \mbf{A} \right) + \frac{m+1}{m} \operatorname{Tr}\left(\mbf{A} \overline{\mbf{\Sigma}}_t \mbf{A}^\top\right), \label{eq:task_test_risk_thm2}
    \end{align}
    where $\mbf{A} = \left( \frac{n+1}{n} \mbf{I}_d + \frac{M_s}{n} \mbf{\Sigma}^{-1} \right) ^{-1}$, $M_s = \tr(\mbf{\Sigma}) + \sigma^2$, and $\mbf{\Sigma} = \sum\limits_{k=1}^K \gamma_k \cdot \mbf{\Sigma}_{s, k}$. 
    \bigskip
    
    \noindent Let $\mbf{U} := \begin{bmatrix}
        \mbf{U}_{s, 1} & \mbf{U}_{s, 2} & \dots & \mbf{U}_{s, K} & \mbf{U}_\perp
    \end{bmatrix}$, where $\mbf{U}_\perp \in \mbb{R}^{d \times (d - Kr)}$ completes the orthonormal basis for $\mbb{R}^d$. By \Cref{lem:I-psd-inv}, 
    \begin{align*}
        \mbf{A} = \left(\frac{n+1}{n} \mbf{I}_d + \frac{M_s}{n} \mbf{\Sigma}^{-1} \right)^{-1} = \mbf{U \Lambda U}^\top,
    \end{align*}
    where 
    \begin{align*}
    \mbf{\Lambda} = \begin{bmatrix}
            \nu_1 \mbf{I}_r & & & \\
            & \ddots & & \\
            & & \nu_K \mbf{I}_r & \\
            & & & \nu_{K+1} \mbf{I}_{d-Kr} 
        \end{bmatrix}
    \end{align*}
    with $\nu_k = \frac{n(\gamma_k + \epsilon)}{(n+1)(\gamma_k + \epsilon) + M_s}$ for all $k \in [K]$, and $\nu_{K+1} = \frac{n \epsilon}{(n+1)\epsilon + r + \epsilon d + \sigma^2}$.
    
    \paragraph{Simplifying $\operatorname{Tr}(\overline{\mbf{\Sigma}}_t)$.} We can write $\operatorname{Tr}(\overline{\mbf{\Sigma}}_t)$ as such:
    \begin{equation*}
        \operatorname{Tr}(\overline{\mbf{\Sigma}}_t) = \operatorname{Tr}(\overline{\mbf{U}}_t \overline{\mbf{U}}_t^\top) + \epsilon \operatorname{Tr}(\mbf{I}_d) = r + \epsilon d.
    \end{equation*}
    
    \paragraph{Simplifying $\operatorname{Tr}(\mbf{A})$ and $\operatorname{Tr}(\mbf{A}^\top \mbf{A})$.} We can write $\operatorname{Tr}(\mbf{A})$ and $\operatorname{Tr}(\mbf{A}^\top \mbf{A})$ as such:
    \begin{align*}
        \tr(\mbf{A}) &= r \sum\limits_{k=1}^K \nu_k + (d-Kr) \nu_{K+1} \quad \text{and} \quad
        \tr(\mbf{A}^\top \mbf{A}) = \tr(\mbf{A}^2) = r\sum\limits_{k=1}^K \nu_k^2 + (d-Kr) \nu_{K+1}^2.
    \end{align*}
    
    \paragraph{Simplifying $\operatorname{Tr}(\overline{\mbf{\Sigma}}_t \mbf{A})$ and $\operatorname{Tr}(\mbf{A} \overline{\mbf{\Sigma}}_t \mbf{A}^\top)$.} Note $\operatorname{Tr}(\mbf{A} \overline{\mbf{\Sigma}}_t \mbf{A}^\top) = \operatorname{Tr}(\overline{\mbf{\Sigma}}_t \mbf{A}^2)$. We first focus on $\operatorname{Tr}(\overline{\mbf{\Sigma}}_t \mbf{A})$:
    \begin{align*}
        &\overline{\mbf{\Sigma}}_t \mbf{A} = \left( \overline{\mbf{U}}_t \overline{\mbf{U}}_t^\top + \epsilon \mbf{I}_d \right) \mbf{U} \mbf{\Lambda} \mbf{U}^\top = \overline{\mbf{U}}_t \overline{\mbf{U}}_t^\top \mbf{U} \mbf{\Lambda} \mbf{U}^\top + \epsilon \mbf{U} \mbf{\Lambda} \mbf{U}^\top \\
        &\implies \operatorname{Tr}\left( \overline{\mbf{\Sigma}}_t \mbf{A} \right) = \operatorname{Tr}\left( \overline{\mbf{U}}_t^\top \mbf{U} \mbf{\Lambda} \mbf{U}^\top \overline{\mbf{U}}_t \right) + \epsilon \operatorname{Tr}(\mbf{A}).
    \end{align*}
    Recall $\overline{\mbf{U}}_t = \sum\limits_{k=1}^K \alpha_k \mbf{U}_{s, k}$ where $\sum\limits_{k=1}^K \alpha_k^2 = 1$, and so we have
    \begin{align*}
        \overline{\mbf{U}}_t^\top \mbf{U} = \left( \sum\limits_{k=1}^K \alpha_k \mbf{U}_k \right)^\top \begin{bmatrix}
            \mbf{U}_{s, 1} & \dots & \mbf{U}_{s,K } & \mbf{U}_{\perp}
        \end{bmatrix} = \begin{bmatrix}
           \alpha_1 \mbf{I}_r & & & \\
           & \ddots & & \\
           & & \alpha_K \mbf{I}_r & \\
           & & & \mbf{0}_{(d - Kr) \times (d - Kr)}
        \end{bmatrix}
    \end{align*}
    Thus,
    \begin{align*}
        &\operatorname{Tr}\left( \overline{\mbf{U}}_t^\top \mbf{U} \mbf{\Lambda} \mbf{U}^\top \overline{\mbf{U}}_t \right) 
        = \operatorname{Tr}\left( \begin{bmatrix}
            \alpha_1^2 \nu_1  \mbf{I}_r & & & \\
            & \ddots & & \\
            & & \alpha_K^2 \nu_K  \mbf{I}_r & \\
            & & & \mbf{0}_{(d - Kr) \times (d - Kr)}
        \end{bmatrix} \right) = r \sum\limits_{k=1}^K \alpha_k^2 \nu_k
    \end{align*}
    Using a similar argument, 
    \begin{equation*}
        \operatorname{Tr}\left( \overline{\mbf{\Sigma}}_t^\top \mbf{A}^2 \right) = r \sum\limits_{k=1}^K \alpha_k^2 \nu_k^2 + \epsilon \operatorname{Tr}\left(\mbf{A}^2\right). 
    \end{equation*}

    \paragraph{Simplifying the test risk.} Substituting the expressions for the $\operatorname{Tr}(\cdot)$ terms into \Cref{eq:task_test_risk_thm2} yields
    \begin{align*}
        \mbb{E}\left[ \left(\widetilde{y} - g^\star_{\mc{W}}(\widetilde{\mbf{Z}})\right)^2 \right] = &\left(\frac{1}{m} \left( r \sum\limits_{k=1}^K \nu_k^2 + (d - Kr) \nu_{K+1}^2 \right) + 1 \right) \big( r + \epsilon d + \sigma^2 \big) \\
        &- 2 \left( r \sum\limits_{k=1}^K \alpha_k^2 \nu_k + \left( r \sum\limits_{k=1}^K \nu_k + (d - Kr) \nu_{K+1} \right)\epsilon  \right) \\
        &+ \frac{m+1}{m} \left( r \sum\limits_{k=1}^K \alpha_k^2 \nu_k^2 + \left( r \sum\limits_{k=1}^K \nu_k^2 + (d - Kr) \nu_{K+1}^2 \right)\epsilon \right).
    \end{align*}
    Taking $\epsilon \to 0$ results in the following expression for the test risk:
    \begin{align*}
        \lim\limits_{\epsilon \to 0} \mbb{E}\left[ \left(\widetilde{y} - g^\star_{\mc{W}}(\widetilde{\mbf{Z}})\right)^2 \right] &= r + \sigma^2 + \frac{(r + \sigma^2) r}{m} \sum\limits_{k=1}^K \left( \frac{\gamma_k n}{\gamma_k(n+1) + r + \sigma^2} \right)^2 \\
        &-2r \sum\limits_{k=1}^K \frac{\alpha_k^2 \gamma_k n}{\gamma_k(n+1) + r + \sigma^2} + \frac{(m+1)r}{m}\sum\limits_{k=1}^K \left(\frac{\alpha_k \gamma_k n}{\gamma_k(n+1) + r + \sigma^2}\right)^2
    \end{align*}
    Substituting $\gamma_k = \frac{1}{K}$ for all $k \in [K]$ and combining like terms yields
    \begin{align*}
        \lim\limits_{\epsilon \to 0} \mbb{E}\left[ \left(\widetilde{y} - g^\star_{\mc{W}}(\widetilde{\mbf{Z}})\right)^2 \right] = r + \sigma^2 + \frac{m + 1 + K(r + \sigma^2)}{m} \cdot \frac{rn^2}{(n + 1 + K(r + \sigma^2))^2} - \frac{2rn}{n + 1 + K(r + \sigma^2)}.
    \end{align*}
    Taking limits, we have
    \begin{align*}
        \lim\limits_{m \to \infty}\lim\limits_{n \to \infty}\lim\limits_{\epsilon \to 0} \mbb{E}\left[ \left(\widetilde{y} - g^\star_{\mc{W}}(\widetilde{\mbf{Z}})\right)^2 \right] = (r+\sigma^2) + r - 2r = \sigma^2, 
    \end{align*}    
 which completes the proof.  
\end{proof}

\subsection{Proof of \Cref{thm:neg_result}}

\begin{proof}
For simplicity, let us denote $\widetilde{y} \coloneqq \widetilde{y}_{m+1}$, and let  $\mbf{U} := \begin{bmatrix}
        \mbf{U}_s & \mbf{U}_{s, \perp} & \mbf{U}_{2q, \perp}
    \end{bmatrix} \in \mbb{R}^{d \times d}$, where $\mbf{U}_s, \mbf{U}_{s, \perp} \in \mbb{R}^{d \times q}$ and $\mbf{U}_{2q, \perp} \in \mbb{R}^{d \times (d - 2q)}$ all have orthonormal columns, while $\mbf{U}_s^\top \mbf{U}_{s, \perp} = \mbf{0}_{q \times q}$ and $ \mbf{U}_s^\top \mbf{U}_{2q, \perp} = \mbf{U}_{s, \perp}^\top \mbf{U}_{2q, \perp} = \mbf{0}_{q \times (d - 2q)}$. We re-write $\mbf{\Sigma}_s$ as such:
    \begin{equation*}
        \mbf{\Sigma}_s = \mbf{U}_s\mbf{U}_s^\top + \epsilon \cdot \mbf{I}_d = \mbf{U} \begin{bmatrix}
            \mbf{I}_q & \ \\
            \ & \mbf{0}_{(d - q) \times (d - q)}
        \end{bmatrix} \mbf{U}^\top + \epsilon \cdot \mbf{I} = \mbf{U} \begin{bmatrix}
            (1 + \epsilon) \mbf{I}_q & \ \\
            \ & \epsilon \mbf{I}_{d - q}
        \end{bmatrix} \mbf{U}^\top.
    \end{equation*}
    Note this is a valid eigendecomposition of $\mbf{\Sigma}_s$. Thus, by Lemma~\ref{lem:I-psd-inv}, we have
    \begin{equation} \label{eq:A-eigen}
        \mbf{A} = \left( \frac{n+1}{n} \mbf{I}_d + \frac{M_s}{n} \mbf{\Sigma}_s^{-1}  \right)^{-1} = \mbf{U} \mbf{\Lambda} \mbf{U}^\top,
    \end{equation}
    where 
    \begin{align*}
        \mbf{\Lambda} = \begin{bmatrix}
            \frac{n(1 + \epsilon)}{(n + 1)(1+\epsilon) + M_s} \cdot \mbf{I}_q & \ \\
            \ & \frac{n \epsilon}{(n+1)\epsilon + M_s} \cdot \mbf{I}_{d - q}
        \end{bmatrix} := \begin{bmatrix}
        \nu_1 \mbf{I}_q & \ \\
        \ & \nu_2 \mbf{I}_{d - q}
        \end{bmatrix}.
    \end{align*}
     and $M_s = \operatorname{Tr}(\mbf{\Sigma}_s) + \sigma^2$. 
    
    \bigskip
    
    \noindent By \cite[Lemma 1]{kwon2026outofdistribution} (and omitting the subscripts in the expectation), 
    \begin{equation} \label{eq:task_test_risk_prop1}
        \mbb{E}\left[ \left(\widetilde{y} - g_{\mc{W}}^\star( \widetilde{\mbf{Z}}) \right)^2 \right] = \left( \frac{1}{m} \operatorname{Tr}\left( \mbf{A}^\top\mbf{A} \right) + 1 \right) \Big(\operatorname{Tr}\left(\mbf{\Sigma}_t\right) + \sigma^2 \Big) - 2 \operatorname{Tr}\left(\mbf{\Sigma}_t \mbf{A} \right) + \frac{m+1}{m} \operatorname{Tr}\left(\mbf{A} \mbf{\Sigma}_t \mbf{A}^\top\right) .
    \end{equation} 
    We simplify the remaining $\operatorname{Tr}(\cdot)$ terms using Equation~(\ref{eq:A-eigen}).

    \paragraph{Simplifying $\operatorname{Tr}(\mbf{A})$ and $\operatorname{Tr}(\mbf{A}^\top \mbf{A})$.} Directly from Equation~(\ref{eq:A-eigen}):
    \begin{align*}
       \operatorname{Tr}(\mbf{A}) = q \cdot \nu_1 + (d - q) \cdot \nu_2 \; \; \text{and} \; \; \operatorname{Tr}(\mbf{A}^\top \mbf{A}) = \operatorname{Tr}(\mbf{A}^2) = q \cdot \nu_1^2 + (d - q) \cdot \nu_2^2,
    \end{align*}
    where $\mbf{A}^2 = \mbf{U} \mbf{\Lambda}^2 \mbf{U}^\top$.

    \paragraph{Simplifying $\operatorname{Tr}(\mbf{\Sigma}_t \mbf{A})$ and $\operatorname{Tr}(\mbf{A} \mbf{\Sigma}_t \mbf{A}^\top)$.} First note $\operatorname{Tr}(\mbf{A} \mbf{\Sigma}_t \mbf{A}^\top) = \operatorname{Tr}(\mbf{\Sigma}_t \mbf{A}^2)$. We first focus on $\operatorname{Tr}(\mbf{\Sigma}_t \mbf{A})$:
    \begin{align*}
        &\mbf{\Sigma}_t \mbf{A} = \left( \mbf{U}_t \mbf{U}_t^\top + \epsilon \mbf{I}_d \right) \mbf{U} \mbf{\Lambda} \mbf{U}^\top = \mbf{U}_t \mbf{U}_t^\top \mbf{U} \mbf{\Lambda} \mbf{U}^\top + \epsilon \mbf{U} \mbf{\Lambda} \mbf{U}^\top \\
        &\implies \operatorname{Tr}\left( \mbf{\Sigma}_t \mbf{A} \right) = \operatorname{Tr}\left( \mbf{U}_t^\top \mbf{U} \mbf{\Lambda} \mbf{U}^\top \mbf{U}_t \right) + \epsilon \operatorname{Tr}(\mbf{A}).
    \end{align*}
    Recall that $\mbf{U}_t$ is defined as follows:
    \begin{equation*}
        \mbf{U}_t = \mbf{U}_s\cos(\mbf{\Theta})   + \mbf{U}_{s,\perp}\sin(\mbf{\Theta}).
    \end{equation*}
    Therefore:
    \begin{align*}
        \mbf{U}_t^\top \mbf{U} = \left(\mbf{U}_s\cos(\mbf{\Theta})   + \mbf{U}_{s,\perp}\sin(\mbf{\Theta}) \right)^\top \begin{bmatrix}
            \mbf{U}_s & \mbf{U}_{s,\perp} & \mbf{U}_{2q, \perp}
        \end{bmatrix} = \begin{bmatrix}
            \cos(\mbf{\Theta})  & \sin(\mbf{\Theta})  & \mbf{0}_{q \times (d-2q)}
        \end{bmatrix},
    \end{align*}
    and thus,
    \begin{align*}
        &\operatorname{Tr}\left( \mbf{U}_t^\top \mbf{U} \mbf{\Lambda} \mbf{U}^\top \mbf{U}_t \right) = \operatorname{Tr} \left( \begin{bmatrix}
            \cos(\mbf{\Theta})  & \sin(\mbf{\Theta})  & \mbf{0}_{q \times (d-2q)}
        \end{bmatrix} \begin{bmatrix}
                \nu_1 \mbf{I}_q & \ & \ \\
                \ & \nu_2 \mbf{I}_q & \ \\
                \ & \ & \nu_2 \mbf{I}_{d - 2q}
            \end{bmatrix} \begin{bmatrix}
            \cos(\mbf{\Theta})  \\ \sin(\mbf{\Theta})  \\ \mbf{0}_{(d-2q) \times q}
        \end{bmatrix} \right) \\
        &= \operatorname{Tr}\left( \begin{bmatrix}
            \nu_1 \cos^2(\mbf{\Theta}) & \ & \ \\
            \ & \nu_2 \sin^2(\mbf{\Theta}) & \ \\
            \ & \ & \mbf{0}_{(d - 2q) \times (d - 2q)}
        \end{bmatrix} \right) = q \cdot \nu_1 \cdot \cos^2(\theta) + q \cdot \nu_2 \cdot \sin^2(\theta),
    \end{align*}
    where we used the fact that the principal angles are all equal to $\theta$.
    Using a similar argument, 
    \begin{equation*}
        \operatorname{Tr}\left( \mbf{\Sigma}_t^\top \mbf{A}^2 \right) = q \cdot \nu_1^2 \cdot \cos^2(\theta) + q \cdot \nu_2^2 \cdot \sin^2(\theta) + \epsilon \operatorname{Tr}(\mbf{A}^2)
    \end{equation*}

    \paragraph{Simplifying the Test Risk.} Substituting the expressions for the $\operatorname{Tr}(\cdot)$ terms into Equation~(\ref{eq:task_test_risk_prop1}) yields
    \begin{align*}
        \mbb{E}\left[ \left(\widetilde{y} - g_{\mc{W}}^\star(\widetilde{\mbf{Z}}) \right)^2 \right] = &\Big( \frac{1}{m} \big( q \nu_1^2 + (d - q) \nu_2^2 \big) + 1 \Big) \big( q + \epsilon d + \sigma^2 \big) \\
        &- 2 \left( q \nu_1 \cos^2(\theta) + q \nu_2 \sin^2(\theta) + \big( q \nu_1 + (d - q) \nu_2 \big)\epsilon  \right) \\
        &+ \frac{m+1}{m} \left( q \nu_1^2 \cos^2(\theta) + q \nu_2^2 \sin^2(\theta) + \big( q \nu_1^2 + (d - q) \nu_2^2 \big) \epsilon \right)
    \end{align*}
    Substituting the expressions for $\nu_1$ and $\nu_2$ and taking $\epsilon \rightarrow 0$ results in the following:
   \begin{align*}
        \lim\limits_{\epsilon \rightarrow 0} \mbb{E}\left[ \left(\widetilde{y} - g_{\mc{W}}^\star(\widetilde{\mbf{Z}}) \right)^2 \right]
        &= \left( \frac{q n^2}{m (n + 1 + q + \sigma^2)^2} + 1 \right)(q + \sigma^2) \\
        &\quad - \frac{2 q n \cos^2(\theta)}{n + 1 + q + \sigma^2} 
        + \frac{(m + 1) q n^2 \cos^2(\theta)}{m (n + 1 + q + \sigma^2)^2}
    \end{align*}  
    Subsequently taking $m, n \rightarrow \infty$ yields
    \begin{align*}
        \lim\limits_{m \rightarrow \infty} \lim\limits_{n \rightarrow \infty} \lim\limits_{\epsilon \rightarrow 0} \mbb{E}&\left[ \left( \widetilde{y} - g_{\mc{W}}^\star(\widetilde{\mbf{Z}}) \right)^2 \right]  = q + \sigma^2 - q \cos^2(\theta) = q \sin^2(\theta) + \sigma^2,
    \end{align*}
    which completes the proof.

\end{proof}

\subsection{Miscellaneous Results}

\begin{lemma} \label{lem:I-psd-inv}
    Let $0 \prec \mbf{\Sigma} \in \mbb{R}^{d \times d}$ and $c, k > 0$ be constants. Then, 
    \begin{equation}
        \left( c \cdot \mbf{I}_d + k \cdot \mbf{\Sigma}^{-1} \right)^{-1} = \mbf{V} \begin{bmatrix}
                \frac{\lambda_1}{c \cdot \lambda_1 + k} & 0 & \dots & 0  \\
                0 & \frac{\lambda_2}{c \cdot \lambda_2 + k} & \dots & 0 \\
                \vdots & \vdots & \ddots & \vdots \\
                0 & 0 & \dots & \frac{\lambda_d}{c \cdot \lambda_d + k}
            \end{bmatrix} \mbf{V}^\top,
    \end{equation}
    where $\mbf{V} \in \mbb{R}^{d \times d}$ is an orthonormal matrix whose columns are eigenvectors of $\mbf{\Sigma}$, and $\lambda_i$ is the $i^{th}$ largest eigenvalue of $\mbf{\Sigma}$.

    \begin{proof}
        Since $\mbf{\Sigma} \succ 0$, there exists an eigendecomposition $\mbf{\Sigma} = \mbf{V} \mbf{\Lambda} \mbf{V}^\top$ such that $\mbf{V}$ is an orthonormal matrix and $\mbf{\Lambda}$ is a diagonal matrix consisting of the real, positive eigenvalues of $\mbf{\Sigma}$, denoted as $\lambda_1, \lambda_2, \dots, \lambda_d$. Thus,
        \begin{align*}
            &\mbf{\Sigma}^{-1} = \mbf{V} \mbf{\Lambda}^{-1} \mbf{V}^\top \implies c \cdot \mbf{I}_d + k \cdot \mbf{\Sigma}^{-1} = \mbf{V}\underbrace{\begin{bmatrix}
                c + \frac{k}{\lambda_1} & 0 & \dots & 0  \\
                0 & c + \frac{k}{\lambda_2} & \dots & 0 \\
                \vdots & \vdots & \ddots & \vdots \\
                0 & 0 & \dots & c + \frac{k}{\lambda_d}
            \end{bmatrix}}_{\widetilde{\mbf{\Lambda}}} \mbf{V}^\top \\
            &\implies \left( c \cdot \mbf{I}_d + k \cdot \mbf{\Sigma}^{-1} \right)^{-1} = \mbf{V} \widetilde{\mbf{\Lambda}}^{-1} \mbf{V}^\top,
        \end{align*}
        which completes the proof.
    \end{proof}
\end{lemma}

\end{document}